\documentclass{article}

\usepackage{arxiv}

\usepackage{hyperref}
\usepackage{url}
\usepackage[utf8]{inputenc} 
\usepackage[T1]{fontenc}    
\usepackage{booktabs}       
\usepackage{amsfonts}       
\usepackage{nicefrac}       
\usepackage{microtype}      
\usepackage{amsthm}
\usepackage{mathtools}
\usepackage{amssymb}
\usepackage{todonotes}

\usepackage[shortlabels]{enumitem}
\usepackage{color}
\usepackage{comment}
\usepackage{subcaption}
\usepackage{graphicx}
\usepackage{booktabs} 
\usepackage{makecell}
\usepackage{multirow}
\usepackage{float}

\newtheorem{theorem}{Theorem}
\newtheorem{definition}[theorem]{Definition}

\newtheorem{lemma}[theorem]{Lemma}
\newtheorem{remark}[theorem]{Remark}

\numberwithin{theorem}{section}

\newcommand{\Lip}{\text{Lip}}

\newcommand{\esssup}{\text{ess\,sup}}

\newcommand{\req}[1]{(\ref{#1})}

\title{Function-space regularized R{\'e}nyi divergences}

\author{  Jeremiah Birrell\\
 Department of Mathematics and Statistics \\
University of Massachusetts Amherst \\
Amherst, MA 01003, USA \\
  \texttt{birrell@math.umass.edu} \\
       \And
 Yannis Pantazis \\
  Institute of Applied and Computational Mathematics\\
  Foundation for Research and Technology - Hellas\\
  Heraklion, GR-70013, Greece \\
  \texttt{pantazis@iacm.forth.gr}\\
   \And
 Paul Dupuis\\
  Division of Applied Mathematics\\
  Brown University\\
  Providence, RI 02912, USA \\
  \texttt{dupuis@dam.brown.edu} \\
  \And
    Markos A. Katsoulakis\\
    Department of Mathematics and Statistics\\
  University of Massachusetts Amherst\\
  Amherst, MA 01003,  USA \\
  \texttt{markos@math.umass.edu} \\
\And
    Luc Rey-Bellet\\
    Department of Mathematics and Statistics\\
  University of Massachusetts Amherst\\
  Amherst, MA 01003,  USA \\
  \texttt{luc@math.umass.edu} 
}

\begin{document}
\maketitle

\begin{abstract}
We propose a new family of regularized Rényi divergences parametrized not only by the order $\alpha$ but also by a variational function space. These new objects are defined by taking the infimal convolution of the standard R\'enyi divergence with the integral probability metric (IPM) associated with the chosen function space. We derive a novel dual variational representation that can be used to construct numerically tractable divergence estimators. This representation avoids risk-sensitive terms and therefore exhibits lower variance, making it well-behaved  when $\alpha>1$; this addresses a notable weakness of prior approaches. We prove several properties of these new divergences, showing that they interpolate between the classical R\'enyi divergences and IPMs. We also study the $\alpha\to\infty$ limit, which leads to a regularized worst-case-regret and a new variational representation in the classical case. Moreover, we show that the proposed regularized R\'enyi divergences inherit features from IPMs such as the ability to compare distributions that are not absolutely continuous, e.g., empirical measures and distributions with low-dimensional support. We present numerical results on both synthetic and real datasets, showing the utility of these new divergences in both estimation and GAN training applications; in particular, we demonstrate significantly reduced variance and improved training performance.
\end{abstract}


\vspace{-2mm}
\section{Introduction}
\vspace{-2mm}
R\'enyi divergence, \cite{Renyi1961}, is a significant extension of Kullback-Leibler (KL) divergence for numerous applications; see, e.g., \cite{van2014renyi}. 
The recent neural-based estimators for divergences \cite{belghazi2018mutual} along with generative adversarial networks (GANs) \cite{GPMXWOCB14} accelerated the use of divergences in the field of deep learning.
The neural-based divergence estimators are feasible through the utilization of variational representation formulas. These formulas are essentially lower bounds (and, occasionally, upper bounds) which are approximated by tractable statistical averages.
The estimation of a divergence based on variational formulas is a notoriously difficult problem. Challenges include  potentially high bias that may require an exponential number of samples \cite{pmlr-v108-mcallester20a} or the exponential statistical variance for certain variational estimators  \cite{2019arXiv191006222S}, rendering  divergence estimation both data inefficient and computationally expensive. This is especially prominent for R\'enyi divergences with order larger than 1. Indeed, numerical simulations have shown that, unless the distributions $P$ and $Q$ are very close to one another, the R\'enyi divergence $R_\alpha(P\|Q)$ is almost intractable  to estimate when $\alpha > 1$  due to the high variance of the statistically-approximated risk-sensitive observables \cite{doi:10.1137/20M1368926}, { see also the recent analysis in  \cite{RenyiCL:2022}}. A similar issue has also been observed for the KL divergence, \cite{2019arXiv191006222S}. { Overall, the lack of estimators with low variance for R\'enyi divergences has prevented wide-spread and accessible experimentation with this class of information-theoretic tools, except in very special cases. We hope our results here  will provide a suitable set of tools to address this gap in the methodology.}

One approach to variance reduction is the development of new variational formulas. This direction is especially fruitful for the estimation of mutual information \cite{DBLP:journals/corr/abs-1807-03748,cheng2020club}. 
Another approach is to regularize the divergence by restricting the function space of the variational formula. Indeed, instead of directly attacking the variance issue, the function space of the variational formula can be restricted, for instance, by bounding the test functions or more appropriately by bounding the derivative of the test functions. The latter regularization leads to Lipschitz continuous function spaces which are also foundational to integral probability metrics (IPMs) and more specifically to the duality property of the Wasserstein metric.  In this paper we combine the above two approaches, first deriving a new variational representation of the classical R{\'e}nyi divergences and then regularizing via an infimal-convolution as follows
\begin{align}\label{eq:Gamma_renyi_intro_def}
    R_\alpha^{\Gamma,IC}(P\|Q)\coloneqq \inf_{\eta}\{R_\alpha(P\|\eta)+W^\Gamma(Q,\eta)\}\,,
\end{align}
where $P$ and $Q$ are the probability distributions being compared, the infimum is over the space of probability measures, $R_\alpha$ is the classical  R{\'e}nyi  divergence, and $W^\Gamma$ is the IPM corresponding to the chosen regularizing function space, $\Gamma$.



The new family of regularized R{\'e}nyi divergences that are developed here address the risk-sensitivity issue inherent in prior approaches.  More specifically, our contributions are as follows.
\begin{itemize}
\item We define a new family of function-space regularized R{\'e}nyi divergences via the infimal convolution operator between the classical Rényi divergence and an arbitrary IPM \req{eq:Gamma_renyi_intro_def}. The new regularized Rényi divergences 
inherit their function space from the IPM. {For instance, they inherit mass transport properties when one regularizes using the 1-Wasserstein metric.}

\item We derive a dual variational representation \req{eq:Renyi_inf_conv} of the regularized R{\'e}nyi divergences which avoids risk-sensitive terms and can therefore be used to construct lower-variance statistical estimators.

\item We prove a series of properties for the new object: (a) the divergence property, (b) being bounded by the minimum of the R{\'e}nyi divergence and IPM, thus allowing for the comparison of non-absolutely continuous distributions, (c) limits as $\alpha\to 1$ from both left and right, (d) regimes in which the limiting cases $R_\alpha(P\|Q)$ and $W^\Gamma(Q,P)$ are recovered.
\item We propose a rescaled version of the regularized Rényi divergences \req{eq:D_alpha_Gamma_IC_def} which lead to a new variational formula for the worst-case regret (i.e., $\alpha\to\infty$). This new variational formula does not involve the essential supremum of the density ratio as in the classical definition of worst-case regret, thereby avoiding risk-sensitive terms.

\item We present a series of illustrative examples and counterexamples that further motivate the proposed definition for the function-space regularized Rényi divergences. 

\item We present numerical experiments that show (a) that we can estimate the new divergence for large values of the order $\alpha$ without variance issues and (b) train GANs using regularized function spaces.

\end{itemize}

{\bf Related work.} 
  The order of R\'enyi divergence controls the weight put on the tails, with the limiting cases being mode-covering and mode-selection \cite{minka2005divergence}. R\'enyi divergence estimation is used in a number of applications, including  \cite{neco_a_01484} (behavioural sciences),  \cite{RDP:Mironov} (differential privacy), and \cite{li2016renyi} (variational inference);  in the latter the variational formula is an adaptation of the evidence lower bound. R\'enyi divergences have been also applied in the training of GANs \cite{Bhatia2020RenyiGA} (loss function for binary classification - discrete case) and in  \cite{CUMGAN} (continuous case, based on the R\'enyi-Donsker-Varahdan variational formula in  \cite{doi:10.1137/20M1368926}).
  R\'enyi divergences with  $\alpha >1$ are also used in contrastive representation learning,
\cite{RenyiCL:2022}, as well as in PAC-Bayesian Bounds, \cite{PAC-Renyi}.
In the context of uncertainty quantification and sensitivity analysis, R\'enyi divergences provide confidence bounds for rare events, \cite{dupuis:renyi, DupuisKatsoulakisPantazisRey-Bellet}, with higher rarity corresponding to larger $\alpha$.

Reducing the variance of divergence estimators through  control of the function space have been recently proposed. In \cite{2019arXiv191006222S} an explicit bound to the output restricts the divergence values. A systematic theoretical framework on how to regularize through the function space has been developed in \cite{Dupuis:Mao,JMLR:v23:21-0100} for the KL and $f$-divergences. Despite not covering the R\'enyi divergence, the theory in \cite{Dupuis:Mao,JMLR:v23:21-0100} and particularly the infimal-convolution formulation clearly inspired the current work. 
%
However, adapting the infimal-convolution method to the R{\'e}nyi divergence setting requires two new technical innovations: 
(a) We develop a new low-variance convex-conjugate variational formula for the classical R{\'e}nyi divergence in Theorem \ref{thm:Renyi_var_LT_main} (see also Fig.~\ref{fig:var_example}), allowing us to apply infimal-convolution tools to develop the new $\Gamma$-R{\'e}nyi divergences in 
Theorem \ref{thm:Renyi_inf_conv}.
(b) We study the $\alpha\to\infty$ limit of (a) to obtain a new low-variance variational representation of worst-case regret in Theorem \ref{thm:worst_case_regret_var} and study its $\Gamma$-regularization in Theorem \ref{thm:alpha_infinity_limit}.

\vspace{-2mm}
\section{New variational representations of  classical R{\'e}nyi divergences}
\vspace{-2mm}
The R{\'e}nyi divergence of order $\alpha\in(0,1)\cup(1,\infty)$ between $P$ and $Q$, denoted $R_\alpha(P\|Q)$, can be defined as follows: Let $\nu$ be a sigma-finite positive measure with  $dP=pd\nu$ and $dQ=qd\nu$. Then
\begin{align}\label{eq:Renyi_formula}
&R_\alpha(P\|Q)=\begin{cases} 
     \frac{1}{\alpha(\alpha-1)}\log\!\left[\int_{q>0} p^\alpha q^{1-\alpha} d\nu\right] &\begin{array}{l}\text{if }0<\alpha<1 \text{ or }\\
\alpha>1 \text{ and } P\ll Q \end{array}\,,\\
       +\infty &\text{ if }\alpha>1\text{ and } P\not\ll Q\,.
         \end{cases}
\end{align}
There always exists such a $\nu$ (e.g., $\nu=P+Q$) and one can show that the definition \req{eq:Renyi_formula} does not depend on the choice of $\nu$. The $R_\alpha$ provide a notion of `distance' between $P$ and $Q$ in that they satisfy the divergence property, i.e., they are non-negative and equal zero iff $Q=P$. The limit of $R_\alpha$ as $\alpha$ approaches $1$ or $0$ equals the KL or reverse KL divergence respectively \cite{van2014renyi}.

Next we present two different variational representations of the R{\'e}nyi divergences, one of which is new.   First we recall the R{\'e}nyi-Donsker-Varadhan variational formula derived in \cite{doi:10.1137/20M1368926}. This is a generalization of the Donsker-Varadhan variational representation of the KL divergence, thus we call it the \textbf{Donsker-Varadhan R{\'e}nyi variational formula} ({\bf DV-R{\'e}nyi}).
\begin{theorem}[Donsker-Varadhan R{\'e}nyi Variational Formula]\label{thm:gen_DV}
Let  $P$ and $Q$ be probability measures on $(\Omega,\mathcal{M})$ and $\alpha\in\mathbb{R}$, $\alpha\neq 0,1$. Then for any set of functions, $\Phi$, with $\mathcal{M}_b(\Omega)\subset\Phi\subset\mathcal{M}(\Omega)$  we have
\begin{align}\label{eq:Renyi_var}
&R_\alpha(P\|Q)=\sup_{\phi\in \Phi}\left\{\frac{1}{\alpha-1}\log\int e^{(\alpha-1)\phi}dP-\frac{1}{\alpha}\log\int e^{\alpha \phi}dQ\right\},
\end{align}
where we interpret $\infty-\infty\equiv-\infty$ and $-\infty+\infty\equiv-\infty$.  

If in addition $(\Omega,\mathcal{M})$ is a metric space with the Borel $\sigma$-algebra then \req{eq:Renyi_var} holds for all $\Phi$ that satisfy $\Lip_b(\Omega)\subset\Phi\subset\mathcal{M}(\Omega)$, where $\Lip_b(\Omega)$ is the space of bounded Lipschitz functions on $\Omega$ (i.e., Lipschitz for any Lipschitz constant $L\in(0,\infty)$). 
\end{theorem}
Here $(\Omega,\mathcal{M})$ denotes a measurable space, $\mathcal{M}(\Omega)$ is the space of measurable real-valued functions on $\Omega$, $\mathcal{M}_b(\Omega)$ is the subspace of bounded functions, and $\mathcal{P}(\Omega)$ denotes be the space of probability measures on $\Omega$.

By a change of variables argument \req{eq:Renyi_var} can be transformed into the following new   variational representation, which we call the {\bf convex-conjugate R{\'e}nyi variational formula} ({\bf CC-R{\'e}nyi}).
\begin{theorem}[Convex-Conjugate R{\'e}nyi Variational Formula]\label{thm:Renyi_var_LT_main}
Let $P,Q\in\mathcal{P}(\Omega)$ and $\alpha\in(0,1)\cup(1,\infty)$.  Then
\begin{align}\label{eq:Renyi_LT_var_Mb_main}
    R_\alpha(P\|Q)=\sup_{g\in \mathcal{M}_b(\Omega):g<0}\left\{\int gdQ+\frac{1}{\alpha-1}\log\int |g|^{(\alpha-1)/\alpha}dP\right\}+\alpha^{-1}(\log\alpha+1)\,.
\end{align}
If $(\Omega,\mathcal{M})$ is a metric space with the Borel $\sigma$-algebra then \req{eq:Renyi_LT_var_Mb_main} holds with  $\mathcal{M}_b(\Omega)$ replaced by $C_b(\Omega)$, the space of bounded continuous real-valued functions on $\Omega$.
\end{theorem}
\begin{proof}
Let $\Phi=\{\alpha^{-1}\log(-h):h\in \mathcal{M}_b(\Omega),h<0\}$.  We have $\mathcal{M}_b(\Omega)\subset\Phi\subset \mathcal{M}(\Omega)$, hence Theorem \ref{thm:gen_DV} implies 
\begin{align}
R_\alpha(P\|Q)=&\sup_{\phi\in \Phi}\left\{\frac{1}{\alpha-1}\log\int e^{(\alpha-1)\phi}dP-\frac{1}{\alpha}\log\int e^{\alpha \phi}dQ\right\}\\
=&\sup_{h\in\mathcal{M}_b(\Omega):h<0}\left\{\frac{1}{\alpha-1}\log\int e^{(\alpha-1)(\alpha^{-1}\log(-h))}dP-\frac{1}{\alpha}\log\int e^{\alpha (\alpha^{-1}\log(-h))}dQ\right\}\notag\\
=&\sup_{h\in\mathcal{M}_b(\Omega):h<0}\left\{\frac{1}{\alpha-1}\log\int |h|^{(\alpha-1)/\alpha}dP-\frac{1}{\alpha}\log\int (-h)dQ\right\}\,.\notag
\end{align}
Note that the second term is finite but the first term is possibly infinite when $\alpha\in(0,1)$. Next use the identity
\begin{align}\label{eq:log_for_LT}
\log(c)=\inf_{z\in\mathbb{R}}\{z-1+ce^{-z}\}\,,\,\,\, c\in(0,\infty)
\end{align}
in the second term to write
\begin{align}
R_\alpha(P\|Q)=&\sup_{h\in\mathcal{M}_b(\Omega):h<0}\left\{\frac{1}{\alpha-1}\log\int |h|^{(\alpha-1)/\alpha}dP-\frac{1}{\alpha}\inf_{z\in\mathbb{R}}\{z-1+e^{-z}\int (-h)dQ\}\right\}\\
=&\sup_{z\in\mathbb{R}}\sup_{h\in\mathcal{M}_b(\Omega):h<0}\left\{\frac{1}{\alpha-1}\log\int |h|^{(\alpha-1)/\alpha}dP+\frac{z-1}{\alpha}+\alpha^{-1}e^{-z}\int hdQ\right\}\,.\notag
\end{align}
For each $z\in\mathbb{R}$ make the change variables $h=\alpha e^z g$, $g\in\mathcal{M}_b(\Omega)$, $g<0$ in the inner supremum  to derive
\begin{align}
R_\alpha(P\|Q)=&\sup_{z\in\mathbb{R}}\sup_{g\in\mathcal{M}_b(\Omega):g<0}\left\{\frac{1}{\alpha-1}\log\int |\alpha e^zg|^{(\alpha-1)/\alpha}dP-\frac{z-1}{\alpha}+\alpha^{-1}e^{-z}\int \alpha e^zgdQ\right\}\\
=&\sup_{z\in\mathbb{R}}\sup_{g\in\mathcal{M}_b(\Omega):g<0}\left\{\frac{1}{\alpha-1}\log\int |g|^{(\alpha-1)/\alpha}dP+( \alpha^{-1}(\log \alpha+1)+\int gdQ)\right\}\notag\\
=&\sup_{g\in\mathcal{M}_b(\Omega):g<0}\left\{\int gdQ+\frac{1}{\alpha-1}\log\int |g|^{(\alpha-1)/\alpha}dP\right\}+\alpha^{-1}(\log \alpha+1)\,.\notag
\end{align}
This completes the proof of \req{eq:Renyi_LT_var_Mb_main}. The proof of the metric-space version in nearly identical.
\end{proof}
\begin{remark}
To reverse the above derivation and obtain \req{eq:Renyi_var} (with $\Phi=\{\phi\in\mathcal{M}(\Omega):\phi\text{ is bounded above}\}$) from \req{eq:Renyi_LT_var_Mb_main}, change variables $g\mapsto -c\exp(\alpha\phi)$, $\phi\in \Phi$, $c>0$ in \req{eq:Renyi_LT_var_Mb_main} and then maximize over $c$.
\end{remark}

 The representation \req{eq:Renyi_LT_var_Mb_main} is of convex-conjugate type, which will be key in our development of function-space regularized R{\'e}nyi divergences. It is also of independent interest as it avoids  risk-sensitive terms, unlike \req{eq:Renyi_var} which contains cumulant-generating-functions. This makes \req{eq:Renyi_LT_var_Mb_main} better behaved in estimation problems, especially when $\alpha>1$; see the example in Section \ref{sec:var_example} below.

We also obtain a new variational formula for  worst-case regret, which is defined by\cite{van2014renyi}
\begin{align}\label{eq:wcr_def_main}
    D_\infty(P\|Q)\coloneqq \lim_{\alpha\to\infty}\alpha R_\alpha(P\|Q)=\begin{cases} 
 \log\left(\esssup_P \frac{dP}{dQ}\right)\,,&P\ll Q\\
\infty\,,&P\not\ll Q\,.
   \end{cases}
\end{align}
In contrast to \req{eq:wcr_def_main}, which requires estimation of the  likelihood ratio, the new  variational formula \req{eq:wcr_var_formula} below avoids risk-sensitive terms.
\begin{theorem}[Convex-Conjugate Worst-case Regret Variational Formula]\label{thm:worst_case_regret_var}
Let $P,Q\in\mathcal{P}(\Omega)$. Then
\begin{align}\label{eq:wcr_var_formula}
    D_\infty(P\|Q)=\sup_{g\in \mathcal{M}_b(\Omega):g<0}\left\{\int gdQ+\log\int|g|dP\right\}+1\,.
\end{align}
If $\Omega$ is a metric space with the Borel $\sigma$-algebra then \req{eq:wcr_var_formula} holds with $\mathcal{M}_b(\Omega)$ replaced by $C_b(\Omega)$.
\end{theorem}
\begin{remark}
Alternative variational formulas for $D_\infty$ on a finite alphabet were derived in \cite{10.1109/ISIT50566.2022.9834358}. In particular, equation (10) in \cite{10.1109/ISIT50566.2022.9834358} can be viewed as a generalization of the Donsker-Varadhan variational formula to worst-case regret, as it is obtained  from taking the limit of the DV-R{\'e}nyi representation \req{eq:Renyi_var}.
\end{remark}
\begin{proof}
First suppose $P\not\ll Q$.  Then there exists a measurable set $A$ with $Q(A)=0$ and $P(A)>0$.  Let $g_n=-n1_A-1_{A^c}$.  Then
\begin{align}
    &\sup_{g\in \mathcal{M}_b(\Omega):g<0}\left\{\int gdQ+\log\int|g|dP\right\}+1\geq\int g_ndQ+\log\int|g_n|dP+1 \\
    =& -nQ(A)-Q(A^c)+\log( nP(A)+P(A^c))+1=\log( nP(A)+P(A^c))\to \infty
\end{align}
as $n\to\infty$.  Therefore
\begin{align}
    \sup_{g\in \mathcal{M}_b(\Omega):g<0}\left\{\int gdQ+\log\int|g|dP\right\}+1=\infty=D_\infty(P\|Q)\,.
\end{align}

Now suppose $P\ll Q$.  Using the definition \req{eq:wcr_def_main} along with Theorem \ref{thm:Renyi_var_LT_main} and changing variables $g= \tilde g/\alpha$ we have
\begin{align}
    D_\infty(P\|Q)=&\lim_{\alpha\to\infty}\alpha R_\alpha(P\|Q)\\
    =&\lim_{\alpha\to\infty}\left[\sup_{g\in \mathcal{M}_b(\Omega):g<0}\left\{\int \alpha gdQ+\frac{\alpha }{\alpha-1}\log\int |g|^{(\alpha-1)/\alpha}dP\right\}+(\log\alpha+1)\right]\notag\\
    \geq&\lim_{\alpha\to\infty}\left[\int  \tilde gdQ+\frac{\alpha }{\alpha-1}\log\int |\tilde g/\alpha|^{(\alpha-1)/\alpha}dP+(\log\alpha+1)\right]\notag\\
    =&\lim_{\alpha\to\infty}\left[\int  \tilde gdQ+\frac{\alpha }{\alpha-1}\log\int |\tilde g| ^{(\alpha-1)/\alpha}dP+1\right]\notag\\
    =&\int \tilde gdQ+\log\int |\tilde g|dP+1\,\,\text{ for all  $\tilde g\in\mathcal{M}_b(\Omega)$, $\tilde g<0$.}\notag
\end{align}
Here we used the dominated convergence theorem to evaluate the limit. Hence, by maximizing over $\tilde g$ we obtain
\begin{align}\label{eq:D_infty_lb}
    D_\infty(P\|Q)\geq \sup_{g\in\mathcal{M}_b(\Omega):g<0}\left\{\int  gdQ+\log\int | g|dP\right\}+1\,.
\end{align}
To prove the reverse inequality, take any $r\in (0,\esssup_PdP/dQ)$. By definition of the essential supremum we have $P(dP/dQ>r)>0$. We also have the bound
\begin{align}
    P(dP/dQ>r)=\int1_{dP/dQ>r} \frac{dP}{dQ} dQ\geq \int1_{dP/dQ>r} r dQ=rQ(dP/dQ>r)\,.
\end{align}
For $c,\epsilon>0$ define $g_{c,\epsilon}=-c1_{dP/dQ>r}-\epsilon$.  These satisfy $g_{c,\epsilon}\in\mathcal{M}_b(\Omega)$, $g_{c,\epsilon}<0$ and so
\begin{align}
    &\sup_{g\in \mathcal{M}_b(\Omega):g<0}\left\{\int gdQ+\log\int|g|dP\right\}+1\geq \int g_{c,\epsilon}dQ+\log\int|g_{c,\epsilon}|dP+1\\
    =&-cQ(dP/dQ>r)-\epsilon+\log(  cP(dP/dQ>r)+\epsilon)+1\notag\\
    \geq &-cP(dP/dQ>r)/r-\epsilon+\log(  cP(dP/dQ>r)+\epsilon)+1\,.\notag
\end{align}
Letting $\epsilon\to 0^+$ we find
\begin{align}
    \sup_{g\in \mathcal{M}_b(\Omega):g<0}\left\{\int gdQ+\log\int|g|dP\right\}+1    \geq &-cP(dP/dQ>r)/r+\log(  cP(dP/dQ>r))+1
\end{align}
 for all $c>0$. We have $P(dP/dQ>r)>0$, hence by maximizing over $c>0$ and changing variables to $z=cP(dP/dQ>r)$ we obtain
\begin{align}
    \sup_{g\in \mathcal{M}_b(\Omega):g<0}\left\{\int gdQ+\log\int|g|dP\right\}+1    \geq &\sup_{z>0}\{-z/r+\log(  z)+1\}=\log(r)\,.
\end{align}
This holds for all $r<\esssup_PdP/dQ$, therefore we can take $r\nearrow  \esssup_PdP/dQ$ and use \req{eq:wcr_def_main} to conclude 
\begin{align}
    \sup_{g\in \mathcal{M}_b(\Omega):g<0}\left\{\int gdQ+\log\int|g|dP\right\}+1    \geq \log(\esssup_PdP/dQ)=D_\infty(P\|Q)\,.
\end{align}
Combining this with \req{eq:D_infty_lb} completes the proof of \req{eq:wcr_var_formula}.

Now suppose $\Omega$ is a metric space. We clearly have
\begin{align}\label{eq:D_inf_var_C_ub}
    D_\infty(P\|Q)=&\sup_{g\in \mathcal{M}_b(\Omega):g<0}\left\{\int gdQ+\log\int|g|dP\right\}+1\\
    \geq& \sup_{g\in C_b(\Omega):g<0}\left\{\int gdQ+\log\int|g|dP\right\}+1\,.\notag
\end{align}
To prove the reverse inequality, take any $g\in\mathcal{M}_b(\Omega)$ with $g<0$.  By Lusin's theorem, for all $\epsilon>0$ there exists a closed set $E_\epsilon$ and $h_\epsilon\in C_b(\Omega)$ such that $P(E^c_\epsilon)\leq \epsilon$, $Q(E^c_\epsilon)\leq \epsilon$, $h_\epsilon|_{E_\epsilon}=g$, and $\inf g\leq h_\epsilon\leq 0$. Define $g_\epsilon=h_\epsilon-\epsilon$.  Then $g_\epsilon<0$, $g_\epsilon\in C_b(\Omega)$ and we have
\begin{align}
    &\sup_{g\in C_b(\Omega):g<0}\left\{\int gdQ+\log\int|g|dP\right\}\geq \int g_\epsilon dQ+\log\int|g_\epsilon|dP\\
    =&\int gdQ+\int (h_\epsilon - g)1_{E^c_\epsilon}dQ-\epsilon+\log(\int|g|dP+\int(|h_\epsilon|-|g|)1_{E_\epsilon^c} dP+\epsilon)\notag\\
    \geq&    \int gdQ-(\sup g-\inf g)Q(E_\epsilon^c)-\epsilon+\log(\int|g|dP+\inf gP(E_\epsilon^c)+\epsilon)\notag\\
        \geq&    \int gdQ-(\sup g-\inf g)\epsilon-\epsilon+\log(\int|g|dP+\inf g\epsilon+\epsilon)\,.\notag
\end{align}
Taking the limit $\epsilon\to 0^+$ we therefore obtain
\begin{align}
    &\sup_{g\in C_b(\Omega):g<0}\left\{\int gdQ+\log\int|g|dP\right\}\geq   \int gdQ+\log\int|g|dP\,.
\end{align}
This holds for all $g\in\mathcal{M}_b(\Omega)$ with $g<0$, hence by taking the supremum over $g$ we obtain the reverse inequality to \req{eq:D_inf_var_C_ub}. This completes the proof.
\end{proof}
Equation \req{eq:wcr_var_formula} is a new result of independent interest and will also be useful in our study of the $\alpha\to\infty$ limit of the function-space regularized R{\'e}nyi divergences  that we define in the next section.

\vspace{-2mm}
\section{Primal and dual formulations of the infimal-convolution $\Gamma$-R{\'e}nyi divergences}\label{sec:primal_dual_formulations}
\vspace{-2mm}
We are now ready to define the function-space regularized R{\'e}nyi divergences and derive their key properties.  In this section, $X$ will denote a compact metric space,   $\mathcal{P}(X)$ will denote the set of Borel probability measures on $X$, and  $C(X)$ will denote the space of  continuous real-valued functions on $X$ (note that $C_b(X)=C(X)$). We equip $C(X)$ with the supremum norm and recall that the dual space of $C(X)$ is   $C(X)^*=M(X)$, the space of finite signed Borel measures on $X$ (see the Riesz representation theorem, e.g., Theorem 7.17 in  \cite{folland2013real}). 

\begin{definition}\label{def:Gamma_Renyi}
Given a test-function space $\Gamma\subset C(X)$, we define the {\bf infimal-convolution $\Gamma$-R{\'e}nyi divergence} (i.e., {\bf IC-$\Gamma$-R{\'e}nyi divergence}) between $P,Q\in\mathcal{P}(X)$ by
\begin{align}\label{eq:Gamma_renyi_def}
    R_\alpha^{\Gamma,IC}(P\|Q)\coloneqq \inf_{\eta\in\mathcal{P}(X)}\{R_\alpha(P\|\eta)+W^\Gamma(Q,\eta)\}\,,\,\,\, \alpha\in(0,1)\cup(1,\infty)\,,
\end{align}
where $W^\Gamma$ denotes the $\Gamma$-IPM
\begin{align}\label{def:IPM}
    W^\Gamma(\mu,\nu)\coloneqq \sup_{g\in\Gamma}\{\int gd\mu-\int gd\nu\}\,, \,\,\,\mu,\nu\in M(X)\,.
\end{align}
\end{definition}
\begin{remark}
The classical R{\'e}nyi divergence is convex in its second argument but not in its first when $\alpha>1$ \cite{van2014renyi}. This is the motivation for defining the IC-$\Gamma$-R{\'e}nyi divergences via an infimal convolution in the second argument of $R_\alpha$; convex analysis tools will be critical in deriving properties of $R_\alpha^{\Gamma,IC}$ below. For $\alpha\in(0,1)$ one can use the  identity $R_{\alpha}(P\|Q)=R_{1-\alpha}(Q\|P)$  to rewrite \req{eq:Gamma_renyi_def} as an infimal convolution in the first argument. 
\end{remark}

The definition \req{eq:Gamma_renyi_def} can be thought of as a regularization of  the classical R{\'e}nyi divergence using the $\Gamma$-IPM.  For computational purposes it is significantly more efficient to have a dual formulation, i.e., a  representation of $R_\alpha^{\Gamma,IC}$ in terms of a supremum over a function space. To derive such a representation  we begin with the  variational formula for $R_\alpha$ from Theorem \ref{thm:Renyi_var_LT_main}. If we define the convex mapping $\Lambda_\alpha^P:C(X)\to(-\infty,\infty]$, 
\begin{align}\label{eq:Lambda_def}
\Lambda_\alpha^P[g]\coloneqq\infty1_{g\not<0}-\left(\frac{1}{\alpha-1}\log\int |g|^{(\alpha-1)/\alpha} dP+\alpha^{-1}(\log\alpha +1)\right)1_{g<0}\,,
\end{align}
then \req{eq:Renyi_LT_var_Mb_main} from Theorem \ref{thm:Renyi_var_LT_main} can be written as a convex conjugate
\begin{align}\label{eq:Renyi_LT}
R_\alpha(P\|Q)=(\Lambda_\alpha^P)^*[Q]\coloneqq \sup_{g\in C(X)}\{\int gdQ-\Lambda_\alpha^P[g]\}\,.
\end{align}
One can then use  Fenchel-Rockafellar duality to derive a dual formulation of the IC-$\Gamma$-R{\'e}nyi divergences. To apply this theory we will need to work with spaces of test functions that satisfy the following admissibility properties. These properties are similar to those used in the construction of regularized KL and $f$-divergences in \cite{Dupuis:Mao} and \cite{JMLR:v23:21-0100}. 
\begin{definition}
We will call $\Gamma\subset C(X)$ {\bf admissible} if it is convex and contains the constant functions.  We will call an admissible $\Gamma$ {\bf strictly admissible} if there exists a $\mathcal{P}(X)$-determining set $\Psi\subset C(X)$ such that for all $\psi\in\Psi$ there exists $c\in\mathbb{R}$, $\epsilon>0$ such that $c\pm \epsilon\psi\in\Gamma$. Recall that $\Psi$ being {\bf $\mathcal{P}(X)$-determining} means that for all $Q,P\in\mathcal{P}(X)$, if $\int \psi dQ=\int \psi dP$ for all $\psi\in\Psi$ then $Q=P$.
\end{definition}

Putting the above pieces together one obtains the following variational representation.
\begin{theorem}\label{thm:Renyi_inf_conv}
Let $\Gamma\subset C(X)$ be admissible, $P,Q\in\mathcal{P}(X)$, and $\alpha\in(0,1)\cup(1,\infty)$. Then:
\begin{enumerate}
\item
\begin{align}\label{eq:Renyi_inf_conv}
R_\alpha^{\Gamma,IC}(P\|Q)=\sup_{g\in \Gamma:g<0}\left\{\int gdQ +\frac{1}{\alpha-1}\log\int |g|^{(\alpha-1)/\alpha} dP\right\}+\alpha^{-1}(\log\alpha +1)\,.
\end{align}
\item If   \req{eq:Renyi_inf_conv} is finite then there exists $\eta_*\in \mathcal{P}(X)$ such that
\begin{align}
R_\alpha^{\Gamma,IC}(P\|Q)=\inf_{\eta\in\mathcal{P}(X)}\{R_\alpha(P\|\eta)+W^\Gamma(Q,\eta)\}=R_\alpha(P\|\eta_*)+W^\Gamma(Q,\eta_*)\,.
\end{align}
\item $R_\alpha^{\Gamma,IC}(P\|Q)$ is convex in $Q$. If $\alpha\in(0,1)$ then $R_\alpha^{\Gamma,IC}(P\|Q)$ is jointly convex in $(P,Q)$.
\item $(P,Q)\mapsto R_\alpha^{\Gamma,IC}(P\|Q)$  is lower semicontinuous.
\item $R_\alpha^{\Gamma,IC}(P\|Q)\geq 0$ with equality if $P=Q$.
\item $R_\alpha^{\Gamma,IC}(P\|Q)\leq \min\{R_\alpha(P\|Q),W^\Gamma(Q,P)\}$.
\item If $\Gamma$ is  strictly admissible then $R_\alpha^{\Gamma,IC}$ has the divergence property.
\end{enumerate}
\end{theorem}
\begin{remark}
We  note that there are alternative strategies for proving  the variational formula \req{eq:Renyi_inf_conv} which make different assumptions; further comments on this can be found in Section \ref{sec:dual_var_proof}.
\end{remark}
\begin{proof}
\begin{enumerate}
\item
Define $F,G:C(X)\to (-\infty,\infty]$ by $F=\Lambda_\alpha^P$ and $G[g]=\infty1_{g\not\in\Gamma}-E_Q[g]$.  Using the assumptions on $\Gamma$ along with  Lemma \ref{lemma:Lambda_cont_app} we see that $F$ and $G$ are convex, $F[-1]<\infty$, $G[-1]<\infty$, and $F$ is continuous at $-1$. Therefore Fenchel-Rockafellar duality (see, e.g., Theorem 4.4.3 in \cite{borwein2006techniques}) along with the identity $C(X)^*=M(X)$ gives
\begin{align}
\sup_{g\in C(X)}\{-F[g]-G[g]\}=\inf_{\eta\in M(X)}\{F^*[\eta]+G^*[-\eta]\}\,,
\end{align}
and if either side is finite then the infimum on the right hand side is achieved at some $\eta_*\in M(X)$. Using the definitions, we can rewrite the left hand side as follow
\begin{align}
&\sup_{g\in C(X)}\{-F[g]-G[g]\}\\
=&\sup_{g\in \Gamma:g<0}\left\{\int gdQ +\frac{1}{\alpha-1}\log\int |g|^{(\alpha-1)/\alpha} dP\right\}+\alpha^{-1}(\log\alpha +1)\,.\notag
\end{align}
We can also compute
\begin{align}
G^*[-\eta]=\sup_{g\in C(X)}\{-\int gd\eta-(\infty 1_{g\not\in\Gamma}-E_Q[g])\}=W^\Gamma(Q,\eta)\,.
\end{align}
Therefore
\begin{align}
&\inf_{\eta\in M(X)}\{(\Lambda_\alpha^P)^*[\eta]+W^\Gamma(Q,\eta)\}\\
=&\sup_{g\in \Gamma:g<0}\left\{\int gdQ +\frac{1}{\alpha-1}\log\int |g|^{(\alpha-1)/\alpha} dP\right\}+\alpha^{-1}(\log\alpha +1)\,.\notag
\end{align}
Next we show that the infimum over $M(X)$ can be restricted to $\mathcal{P}(X)$. First suppose $\eta\in M(X)$ with $\eta(X)\neq 1$.  Then, using the assumption that $\Gamma$ contains the constant functions, we have
\begin{align}
W^\Gamma(Q,\eta)\geq  E_Q[\pm n]-\int \pm nd\eta=\pm n(1- \eta(X))\to \infty
\end{align}
as $n\to\infty$ (for appropriate choice of sign).  Therefore $W^\Gamma(Q,\eta)=\infty$ if $\eta(X)\neq 1$.  This implies that the infimum can be restricted to $\{\eta\in 
M(X):\eta(X)=1\}$.

Now suppose $\eta\in M(X)$ is not positive.  Take a measurable set $A$ with $\eta(A)<0$. By Lusin's theorem, for all $\epsilon>0$ there exists a closed set $E_\epsilon\subset X$ and a continuous function $g_\epsilon\in C(X)$ such that $|\eta|(E_\epsilon^c)<\epsilon$, $0\leq g_\epsilon\leq 1$, and $g_\epsilon|_{E_\epsilon}=1_A$.  Define $g_{n,\epsilon}=-ng_\epsilon-1$, $n\in\mathbb{Z}^+$.  Then $g_{n,\epsilon}\in\{g\in C(X):g<0\}$, hence
\begin{align}
(\Lambda_\alpha^P)^*[\eta]\geq &\int g_{n,\epsilon}d\eta+\frac{1}{\alpha-1}\log\int |g_{n,\epsilon}|^{(\alpha-1)/\alpha} dP+\alpha^{-1}(\log\alpha +1)\\
=&-n\eta(A)+n \eta(A\cap E_\epsilon^c)-n\int g_\epsilon 1_{E_\epsilon^c}d\eta-\eta(X)\notag\\
&+\frac{1}{\alpha-1}\log\int |ng_\epsilon+1|^{(\alpha-1)/\alpha} dP+\alpha^{-1}(\log\alpha +1)\notag\\
\geq&n(|\eta(A)|-2\epsilon)-\eta(X)+\frac{1}{\alpha-1}\log\int |ng_\epsilon+1|^{(\alpha-1)/\alpha} dP+\alpha^{-1}(\log\alpha +1)\,.\notag
\end{align}
If  $\alpha>1$ then $\log\int |ng_\epsilon+1|^{(\alpha-1)/\alpha} dP\geq 0$ and if   $\alpha\in(0,1)$ then $\log\int|ng_\epsilon+1|^{(\alpha-1)/\alpha}dP\leq 0$. In either case we have $\frac{1}{\alpha-1}\log\int |ng_\epsilon+1|^{(\alpha-1)/\alpha} dP\geq 0$ and so
\begin{align}
(\Lambda_\alpha^P)^*[\eta]\geq &n(|\eta(A)|-2\epsilon)-\eta(X)+\alpha^{-1}(\log\alpha +1)\,.
\end{align}
By choosing $\epsilon<|\eta(A)|/2$ and taking $n\to\infty$ we see that $(\Lambda_\alpha^P)^*[\eta]=\infty$ whenever $\eta\in M(X)$ is not positive. Therefore the infimum can further be restricted to positive measures. Combining these results we find 
\begin{align}
&\sup_{g\in \Gamma:g<0}\left\{\int gdQ +\frac{1}{\alpha-1}\log\int |g|^{(\alpha-1)/\alpha} dP\right\}+\alpha^{-1}(\log\alpha +1)
\\
=&\inf_{\eta\in \mathcal{P}(X)}\{(\Lambda_\alpha^P)^*[\eta]+W^\Gamma(Q,\eta)\}\,.\notag
\end{align}
For $\eta\in\mathcal{P}(X)$, equation  \req{eq:Renyi_LT} implies $(\Lambda_\alpha^P)^*[\eta]=R_\alpha(P\|\eta)$. This completes the proof.
\item The existence of a minimizer follows from Fenchel-Rockafellar duality; again, see Theorem 4.4.3 in \cite{borwein2006techniques}.

\item This follows from \req{eq:Renyi_inf_conv} together with the fact that the supremum of convex functions is convex and $y\mapsto \frac{1}{\alpha-1}\log(y)$ is convex when $\alpha\in(0,1)$.

\item Compactness of $X$ implies that $g$ and $|g|^{(\alpha-1)/\alpha}$ are bounded and continuous whenever $g\in \Gamma$ satisfies $g<0$.  Therefore $Q\to \int gdQ$ and $P\to\int |g|^{(\alpha-1)/\alpha}dP$ are continuous in the weak topology on $\mathcal{P}(X)$. Therefore the objective functional in \req{eq:Renyi_inf_conv} is continuous in $(P,Q)$.  The supremum is therefore lower semicontinuous.

\item This easily follows from the definition \req{eq:Gamma_renyi_def}.

\item $R_\alpha$ is a divergence, hence is non-negative. $\Gamma$ contains the constant functions, hence $W^\Gamma\geq 0$.  Therefore $R_\alpha^{\Gamma,IC}\geq 0$.  If $Q=P$ then $0\leq R_\alpha^{\Gamma,IC}(P\|Q)\leq R_\alpha(P\|P)+W^\Gamma(P,P)=0$, hence $R_\alpha^{\Gamma,IC}(P\|Q)=0$.
\item Suppose $\Gamma$ is strictly admissible. Due to part 5 of this theorem, we only need to show that if $R_\alpha^{\Gamma,IC}(P\|Q)=0$ then $P=Q$.   If $R_\alpha^{\Gamma,IC}(P\|Q)=0$ then part 2 implies there exists $\eta_*\in\mathcal{P}(X)$ such that 
\begin{align}
0=R_\alpha(P\|\eta_*)+W^\Gamma(Q,\eta_*)\,.
\end{align}
Both terms are non-negative, hence $R_\alpha(P\|\eta_*)=0=W^\Gamma(Q,\eta_*)$. $R_\alpha$ has the divergence property, hence $\eta_*=P$.  So $W^\Gamma(Q,P)=0$. Therefore $0\geq \int gdQ-\int gdP$ for all $g\in\Gamma$.  Let $\Psi$ be as in the definition of strict admissibility and let $\psi\in\Psi$.  There exists $c\in\mathbb{R}$, $\epsilon>0$ such that $c\pm\epsilon\psi\in\Gamma$ and so $0\geq  \pm\epsilon(\int\psi dQ-\int\psi dP)$.  Therefore $\int \psi dQ=\int \psi dP$ for all $\psi\in\Psi$.  $\Psi$ is $\mathcal{P}(X)$-determining, hence $Q=P$.
\end{enumerate}

\end{proof}
 Important examples of strictly admissible $\Gamma$ include the following:
\begin{enumerate}
    \item $\Gamma=C(X)$, which leads to the classical R{\'e}nyi-divergences.
    \item $\Gamma=\Lip^1(X)$, i.e. all  1-Lipschitz functions. This regularizes the R{\'e}nyi divergences via the Wasserstein metric.
    \item $\Gamma=\{c+g:c\in\mathbb{R},g\in C(X),|g|\leq 1\}$. This regularizes the R{\'e}nyi divergences via the total-variation  metric.
        \item $\Gamma=\{c+g:c\in\mathbb{R},g\in \Lip^1(X),|g|\leq 1\}$. This regularizes the R{\'e}nyi divergences via the Dudley metric.
                \item  $\Gamma=\{c+g:c\in\mathbb{R},g\in Y:\|g\|_V\leq 1\}$, the unit ball in a RKHS $V\subset C(X)$.  This regularizes the R{\'e}nyi divergences via  MMD.
\end{enumerate}

The IC-$\Gamma$-R{\'e}nyi divergences also satisfy a data processing inequality. First we introduce the following notation: Let $Y$ be another compact metric space and $K$ be a probability kernel from $X$ to $Y$. Given $P\in\mathcal{P}(X)$ we denote the composition of $P$ with $K$ by $P\otimes K$ (a probability measure on $X\times Y$) and we denote the marginal distribution on $Y$ by $K[P]$. Given $g\in C(X\times Y)$ we let $K[g]$ denote the   function on $X$ given by $x\mapsto\int g(x,y)K_x(dy)$. See Theorem \ref{thm:IC_data_processing_app} in Section \ref{app:proofs} for a proof.
\begin{theorem}[Data Processing Inequality]\label{thm:IC_data_processing}
Let $\alpha\in(0,1)\cup(1,\infty)$, $Q,P\in\mathcal{P}(X)$, and $K$ be a probability kernel from $X$ to $Y$ such that $K[g]\in C(X)$ for all $g\in C(X,Y)$.  
\begin{enumerate}
    \item  If  $\Gamma\subset C(Y)$ is admissible then $R_\alpha^{\Gamma,IC}\left(K[P]\|K[Q]\right)\leq R_\alpha^{K[\Gamma],IC}(P\|Q)$.
    \item  If  $\Gamma\subset C(X\times Y)$ is admissible then $R_\alpha^{\Gamma,IC}\left(P\otimes K\|Q\otimes K\right)\leq R_\alpha^{K[\Gamma],IC}(P\|Q)$.
\end{enumerate}
\end{theorem}
Note that if $K[\Gamma]$ is strictly contained in $\Gamma$ then the bounds in Theorem \ref{thm:IC_data_processing} can be strictly tighter than the classical data processing inequality \cite{van2014renyi}.  Data-processing inequalities are important for constructing symmetry-preserving GANs; see \cite{pmlr-v162-birrell22a} and Section \ref{sec:RotMNIST}.

\vspace{-2mm}
\section{Limits, interpolations, and regularized worst-case regret}\label{sec:limit_properties}
\vspace{-2mm}
Next we use Theorem \ref{thm:Renyi_inf_conv} to compute various limits of the IC-$\Gamma$-R{\'e}nyi divergences. First we show that they interpolate between $R_\alpha$ and $W^\Gamma$ in the following sense (see Theorem \ref{thm:IC_limits_app}  for a proof).
\begin{theorem}\label{thm:IC_limits}
Let $\Gamma\subset C(X)$ be admissible, $P,Q\in\mathcal{P}(X)$, and $\alpha\in(0,1)\cup(1,\infty)$.  
\begin{enumerate}
    \item $\lim_{\delta\to0^+}\frac{1}{\delta}R_\alpha^{\delta\Gamma,IC}(P\|Q)=W^\Gamma(Q,P)$,
\item If $\Gamma$ is strictly admissible then $\lim_{L\to\infty}R_\alpha^{L\Gamma,IC}(P\|Q)=R_\alpha(P\|Q)$.
\end{enumerate}
\end{theorem}

Now we discuss the limiting behavior in $\alpha$. These results generalize several properties of the classical R{\'e}nyi divergences \cite{van2014renyi}. First we consider the $\alpha\to 1$ limit; see Theorem \ref{thm:alpha_1_limit_app} for a proof.
\begin{theorem}\label{thm:alpha_1_limit}
Let $\Gamma\subset C(X)$ be admissible and $P,Q\in\mathcal{P}(X)$.  Then
\begin{align}
\lim_{\alpha\to1^+}R_\alpha^{\Gamma,IC}(P\|Q)=&\inf_{\substack{\eta\in\mathcal{P}(X):\\\exists\beta>1,R_\beta(P\|\eta)<\infty}}\{R(P\|\eta)+W^\Gamma(Q,\eta)\}\,,\label{eq:IC_rev_KL_plus}\\
    \lim_{\alpha\to 1^-}R_\alpha^{\Gamma,IC}(P\|Q)=&\inf_{\eta\in\mathcal{P}(X)}\{R(P\|\eta)+W^\Gamma(Q,\eta)\}\label{eq:IC_rev_KL}\\
    =&\sup_{g\in \Gamma:g<0}\{\int gdQ+\int \log|g|dP\}+1\,.\label{eq:rev_KL_var_formula_minus}
\end{align}
\end{theorem}
\begin{remark}
When $\Gamma=C(X)$, changing variables to $g=-\exp(\phi-1)$ transforms \req{eq:rev_KL_var_formula_minus} into the Legendre-transform variational formula for $R(P\|Q)$; see equation (1) in \cite{9737725} with $f(x)=x\log(x)$. Also note that  \req{eq:IC_rev_KL} is an infimal convolution of the reverse KL-divergence, as opposed to the results in \cite{Dupuis:Mao} which apply to the (forward) KL-divergence.
\end{remark}

{\bf Function-space regularized worst-case regret.} 
Next we investigate the $\alpha\to\infty$ limit of the IC-$\Gamma$-R{\'e}nyi divergences, which will lead to the function-space regularized worst-case regret.   First recall that some authors use an alternative definition of the classical  R{\'e}nyi divergences, related to the one used in this paper by $D_\alpha(\cdot\|\cdot)\coloneqq\alpha R_\alpha(\cdot\|\cdot)$. This alternative definition has the useful property of being non-decreasing in $\alpha$; see \cite{van2014renyi}.  Appropriately rescaled, the IC-$\Gamma$-R{\'e}nyi divergence also satisfies this property, leading to the following definition.
\begin{definition}\label{def:D_alpha_Gamma_IC}
For $\Gamma\subset C(X)$, $\alpha\in(0,1)\cup(1,\infty)$ and $P,Q\in\mathcal{P}(X)$ we define
\begin{align}\label{eq:D_alpha_Gamma_IC_def}
    D_\alpha^{\Gamma,IC}(P\|Q)\coloneqq \alpha R_\alpha^{\Gamma/\alpha,IC}(P\|Q)\,.
\end{align}
\end{definition}
Note that $\alpha R_\alpha^{\Gamma/\alpha,IC}(P\|Q)$ is non-decreasing in $\alpha$; see Lemma \ref{lemma:R_IC_nondec_app} for a proof. We now show that the divergences $D_\alpha^{\Gamma,IC}$ are well behaved in the $\alpha\to\infty$ limit, generalizing \req{eq:wcr_def_main}. Taking this limit provides a definition of function-space regularized worst-case regret, along with the following dual variational representation.
\begin{theorem}\label{thm:alpha_infinity_limit}  Let $\Gamma\subset C(X)$ be admissible and $P,Q\in\mathcal{P}(X)$. Then
\begin{align}
   D_\infty^{\Gamma,IC}(P\|Q)\coloneqq& \lim_{\alpha\to\infty}D_\alpha^{\Gamma,IC}(P\|Q)=\inf_{\eta\in P(X)}\{D_\infty(P\|\eta)+W^\Gamma(Q,\eta)\}\label{eq:D_inf_Gamma_def}\\
    =&\sup_{g\in \Gamma:g<0}\left\{\int gdQ+\log\int|g|dP\right\}+1\,.\label{eq:D_inf_Gamma_var_formula}
\end{align}
\end{theorem}
We call $D_\infty^{\Gamma,IC}$ the {\bf infimal-convolution $\Gamma$-worst-case regret} (i.e., {\bf IC-$\Gamma$-WCR}). The method of  proof of Theorem \ref{thm:alpha_infinity_limit} is similar to that of part (1) of Theorem \ref{thm:Renyi_inf_conv}; see Theorem \ref{thm:alpha_infinity_limit_app} in Section \ref{app:proofs} for details. Theorem \ref{thm:alpha_infinity_limit} suggests that  $D_\alpha^{\Gamma,IC}$ is the appropriate $\alpha$-scaling to use when $\alpha$ is large and we find this to be the case in practice; see the example in Section \ref{ex:CIFAR10}.

\vspace{-2mm}
\section{Analytical examples and counterexamples}\label{sec:examples}
\vspace{-2mm}
In this section we present several  analytical examples and counterexamples that illustrate important properties of the IC-$\Gamma$-R{\'e}nyi divergences and demonstrate weaknesses of other attempts to define regularized R{\'e}nyi divergences. 

\vspace{2mm}
{\bf Infimal convolution and scaling limits:}
First we present a simple example that illustrates the infimal convolution formula and limiting properties from Sections \ref{sec:primal_dual_formulations} and \ref{sec:limit_properties}.

Let $P=\delta_0$, $Q_{x,c}=c\delta_0+(1-c)\delta_x$ for $c\in(0,1)$, $x>0$, and let $\Gamma=\Lip^1$.  Then for $L>0$ one can compute
\begin{align}
    R_\alpha(P\|Q_{x,c})=&\alpha^{-1}\log(1/c)\,,\\
    W^{L\Gamma}(Q_{x,c},P)=&(1-c)Lx\,,
\end{align}
and
\begin{align}
    R_\alpha^{L\Gamma,IC}(P\|Q_{x,c})=&\sup_{a,b<0:|a-b|\leq x}\{Lca+L(1-c)b+\alpha^{-1}\log(L|a|)\}+\alpha^{-1}(\log\alpha+1)\\
    =&\sup_{a<0}\{Lca+L(1-c)\min\{x+a,0\}+\alpha^{-1}\log(L|a|)\}+\alpha^{-1}(\log\alpha+1)\notag\\
    =&\alpha^{-1}+\alpha^{-1}\sup_{y>0}\begin{cases}
    -cy+\log y\,,\,\,\, y\leq \alpha Lx\\
    (1-c)\alpha Lx-y+\log y \,,\,\,\,y>\alpha Lx
    \end{cases}
    \notag\\
    =&\begin{cases}
    (1-c)Lx\,,\,\,\,0<\alpha Lx<1\\
    \alpha^{-1}-cLx+\alpha^{-1}\log(\alpha Lx)\,,\,\,\,1\leq \alpha Lx\leq 1/c\\
    \alpha^{-1}\log(1/c)\,,\,\,\,\alpha Lx>1/c
    \end{cases}\,.\notag
    \end{align}
    In particular, it is straightforward to show that
    \begin{align}
        &R_\alpha^{L\Gamma,IC}(P\|Q_{x,c})\leq W^{L\Gamma}(Q_{x,c},P)\,,\\
        &\lim_{x\to 0^+}R_{\alpha}^{L\Gamma,IC}(P\|Q_{x,c})=\lim_{x\to 0^+}(1-c)Lx=0\,,\\ &\lim_{L\to\infty}R_\alpha^{L\Gamma,IC}(P\|Q_{x,c})=\alpha \log(1/c)=R_\alpha(P\|Q_{x,c})\,.
        \end{align}
    We can also rewrite this in terms of the solution to the infimal convolution problem as follows
    \begin{align}
    R_\alpha^{L\Gamma,IC}(P\|Q_{x,c})=&\begin{cases}
    W^{L\Gamma}(Q_{x,c},P)\,,\,\,\,0<\alpha Lx<1\\
    R_\alpha(P\|Q_{x,1/(\alpha Lx)})+W^{L\Gamma}(Q_{x,c},Q_{x,1/(\alpha Lx)})\,,\,\,\,1\leq \alpha Lx\leq 1/c\\
   R_\alpha(P\|Q_{x,c})\,,\,\,\,\alpha Lx>1/c
    \end{cases}\,.
\end{align}
Taking the worst-case-regret scaling limit we find
\begin{align}
    \lim_{\alpha\to\infty}\alpha R_\alpha^{\Gamma/\alpha,IC}(P\|Q_{x,c})=&\begin{cases}
    (1-c)x\,,\,\,\,0<x<1\\
    1-cx+\log(x)\,,\,\,\,1\leq x\leq 1/c\\
    \log(1/c)\,,\,\,\,x>1/c
    \end{cases}\\
    =&\begin{cases}
    W^{\Gamma}(Q_{x,c},P)\,,\,\,\,0<x<1\\
    D_\infty(P\|Q_{x,1/x})+W^{\Gamma}(Q_{x,c},Q_{x,1/x})\,,\,\,\,1\leq x\leq 1/c\\
   D_\infty(P\|Q_{x,c})\,,\,\,\,x>1/c
    \end{cases}\,,\notag
\end{align}
where $D_\infty(P\|Q_{x,c})=\log(1/c)$.

{\bf $\Gamma$-R{\'e}nyi-Donsker-Varadhan counterexample:} 
As an alternative to Definition \ref{def:Gamma_Renyi}, one can attempt to regularize the R{\'e}nyi divergences by restricting the test-function space in the variational representation \req{eq:Renyi_var}, leading to the {\bf $\Gamma$-R{\'e}nyi-Donsker-Varadhan} divergences
\begin{align}\label{eq:R_Gamma_DV}
    R^{\Gamma,DV}_\alpha(P\|Q)\coloneqq\sup_{\phi\in \Gamma}\left\{\frac{1}{\alpha-1}\log\int e^{(\alpha-1)\phi}dP-\frac{1}{\alpha}\log\int e^{\alpha \phi}dQ\right\}\,.
\end{align}
The bound $\log \int e^{c\phi}dP\geq c \int \phi dP$ for all $\phi\in\Gamma,c\in\mathbb{R}$ implies that $R_\alpha^{\Gamma,DV}\leq W^\Gamma$ for $\alpha\in(0,1)$, making \req{eq:R_Gamma_DV} a useful regularization of the R{\'e}nyi divergences in this case; this utility was  demonstrated in \cite{CUMGAN}, where it was used to construct GANs.  However, estimators built from the representation \req{eq:R_Gamma_DV} (i.e., replacing $P$ and $Q$ by empirical measures) are known to be numerically unstable when $\alpha>1$.  Below we provide a counterexample showing that, unlike for the IC-$\Gamma$-R{\'e}nyi divergences, $R^{\Gamma,DV}_\alpha\not\leq W^\Gamma$ in general when $\alpha>1$.  We conjecture that this  is a key reason for the instability of $\Gamma$-R{\'e}nyi-Donsker-Varadhan estimators when $\alpha>1$.

Let $P_{x,c}=c\delta_0+(1-c)\delta_x$, $Q=\delta_0$ for $x>0$, $c\in(0,1)$ and  $\Gamma_L=\Lip^L$.  Then for $\alpha>1$ we have
\begin{align}
    R_\alpha^{\Gamma_L,DV}(P_{x,c}\|Q)=&\sup_{a,b\in\mathbb{R}:|a-b|\leq Lx} \left\{\frac{1}{\alpha-1}\log(c\exp((\alpha-1)a)+(1-c)\exp((\alpha-1)b))-a\right\}\label{eq:R_DV_counterexample_app}\\
    =&\sup_{a\in\mathbb{R}}\left\{\frac{1}{\alpha-1}\log(c\exp((\alpha-1)a)+(1-c)\exp((\alpha-1)(Lx+a)))-a\right\}\notag\\
    =&\frac{1}{\alpha-1}\log\left(c+(1-c)\exp((\alpha-1)Lx)\right)\,,\notag\\
    W^{\Gamma_L}(P_{x,c},Q)=&\sup_{|a-b|\leq Lx}\{ca+(1-c)b-a\}=(1-c)Lx\,.
\end{align}
Note that the condition $\alpha>1$ was crucial in computing the supreumum over $b$ in  \req{eq:R_DV_counterexample_app}. Using strict concavity of the logarithm  one can then obtain the bound
\begin{align}
    R_\alpha^{\Gamma_L,DV}(P_{x,c}\|Q)>W^{\Gamma_L}(P_{x,c},Q)\,.
\end{align}
This shows that, when $\alpha>1$, $\Gamma$-R{\'e}nyi-DV  violates the key property that allows the IC-$\Gamma$-R{\'e}nyi divergences to inherit properties from the corresponding $\Gamma$-IPM.

{\bf $\log$-$\Gamma$-R{\'e}nyi-Donsker-Varadhan counterexample:} 
A second alternative to Definition \ref{def:Gamma_Renyi} is to again start with \req{eq:Renyi_var} and then reduce the test-function space to $\frac{1}{\alpha}\log(\Gamma)$
\begin{align}\label{eq:DV_log_app}
{R}^{\Gamma, log-DV}_\alpha(P\|Q)\coloneqq\sup_{g\in\Gamma:g>0}\left\{\frac{1}{\alpha-1}\log\int g^{(\alpha-1)/\alpha} dP-\frac{1}{\alpha}\log\int g dQ\right\}\,.
\end{align}
However, as we show below, this definition fails to provide a regularized divergence; in particular, it is incapable of meaningfully comparing Dirac distributions.

Let $P=\delta_0$, $Q_x=\delta_x$, $x>0$, $\Gamma_L=\Lip^L$.  Then straightforward computations using the variational definition gives
\begin{align}
R_\alpha^{\Gamma_L,DV-log}(P\|Q_x)=&\alpha^{-1}\sup_{\phi\in\Gamma,\phi>0}\log(\phi(0)/\phi(x))\\
=&\alpha^{-1}\sup_{b>0}\sup_{a>0:b-x\leq a\leq x+b}\log(a/b)\notag\\
=&\alpha^{-1}\sup_{b>0}\log(1+x/b)=\infty\,.\notag
\end{align}
In contrast we have
\begin{align}
&R_\alpha^{\Gamma_L,IC}(P\|Q_x)\\
=&\sup_{a<0,b<0:b-x\leq a\leq b+x}\{Lb+\alpha^{-1}\log L+\alpha^{-1}\log(|a|)\}+\alpha^{-1}(\log(\alpha)+1)\notag\\
=&\sup_{b<0}\{Lb+\frac{1}{\alpha}\log(|b-x|)\}+\alpha^{-1}\log L+\alpha^{-1}(\log\alpha+1)\notag\\
=&\begin{cases}
\alpha^{-1}\log(\alpha Lx)+\alpha^{-1}\,, & x\geq 1/(\alpha L)\\
Lx\,,& x<1/(\alpha L)
\end{cases}\,.\notag
\end{align}
In particular, 
\begin{align}
    &R_\alpha^{\Gamma_L,IC}(P\|Q_x)\leq Lx=W^{\Gamma_L}(P,Q_x)\,,\\
&\lim_{x\to 0^+}R_\alpha^{\Gamma_L,IC}(P\|Q_x)=0\,,
\end{align}

showing that $R_\alpha^{\Gamma_L,IC}$ is able to capture the convergence of $Q_x$ to $P$ as $x\to 0^+$, while  $R_\alpha^{\Gamma,log-DV}$ fails to do so. The above two counterexamples  lend further credence to our infimal-convolution based regularization approach \req{eq:Gamma_renyi_def}.

\vspace{-2mm}
\section{Numerical experiments}\label{sec:experiments}
\vspace{-2mm}

In this section we present  numerical examples that demonstrate the use of the IC-$\Gamma$-R{\'e}nyi divergences for both estimation and  training of GANs. All of the divergences considered in this paper have a  variational representation of the form $D(P\|Q)=\sup_{g\in\Gamma} H[g;P,Q]$ for some objective functional $H$; we use the corresponding estimator
\begin{align}\label{eq:div_estimator}
    \widehat{D}_n(P\|Q)\coloneqq\sup_{\theta\in\Theta} H[g_\theta;P_n,Q_n]
\end{align}
where $P_n$, $Q_n$ are $n$-sample empirical measures and $g_\theta$ is a family of neural networks (NN) with parameters $\theta\in\Theta$. For Lipschitz function spaces we weaken the Lipschitz constraint to a soft 1-sided gradient penalty (see Section 4.1 of \cite{JMLR:v23:21-0100}). Optimization is performed using the Adam optimizer \cite{kingma2014adam}. For the infimal convolution divergences we enforce negativity of the test function (i.e., discriminators) using a  final layer having one of the following forms: 1)$ -abs(x)$ or 2) $-(1/(1-x) 1_{x<0}+(1+x)1_{x\geq 0})$.  The latter, which we term poly-softplus, is $C^1$ and decays like $O(x^{-1})$ as $x\to-\infty$. Table \ref{table:experiments} gives an overview of the divergences studied in the  experiments below.
\begin{table}[h]
\centering
\caption{Overview of numerical experiments.}
\begin{tabular}{|c c c |} 
 \hline
 Divergence & Definition  & Related Experiments\\
 \hline
 \hline
DV-R{\'e}nyi & \req{eq:Renyi_var} & Section \ref{sec:var_example}\\
CC-R{\'e}nyi & \req{eq:Renyi_LT_var_Mb_main} & Section \ref{sec:var_example}\\
IC-$\Gamma$-R{\'e}nyi& \req{eq:Gamma_renyi_def} and \req{eq:Renyi_inf_conv} &Sections \ref{ex:CIFAR10} and \ref{sec:RotMNIST} \\
Rescaled IC-$\Gamma$-R{\'e}nyi & \req{eq:D_alpha_Gamma_IC_def} & Sections \ref{sec:subpop} and \ref{ex:CIFAR10}\\
DV-WCR & Equation (10) in \cite{10.1109/ISIT50566.2022.9834358} & Section \ref{sec:var_example}\\
CC-WCR & \req{thm:worst_case_regret_var} & Section \ref{sec:var_example}\\
IC-$\Gamma$-WCR &\req{eq:D_inf_Gamma_def}-\req{eq:D_inf_Gamma_var_formula}& Sections \ref{sec:subpop} and \ref{ex:CIFAR10}\\
 \hline
 \end{tabular}\label{table:experiments}
 \end{table}

\subsection{Variance of R{\'e}nyi estimators}\label{sec:var_example}
 As a first example, we compare estimators of the classical R{\'e}nyi divergences  (i.e., without regularization) constructed from DV-R{\'e}nyi \req{eq:Renyi_var} and CC-R{\'e}nyi \req{eq:Renyi_LT_var_Mb_main} in a simple case where the exact R{\'e}nyi divergence is known. We let $Q$ and $P$ be 1000-dimensional Gaussians with equal variance and study $R_\alpha(P\|Q)$ as a function of the separation between their means.  The results are shown in Figure \ref{fig:var_example}. We see that the estimator based on the convex-conjugate R{\'e}nyi variational formula  \ref{eq:Renyi_LT_var_Mb_main} has smaller variance and mean-squared error (MSE) that the R{\'e}nyi-Donsker-Varadhan variational formula \ref{eq:Renyi_var}, with the difference becoming very large when $\alpha\gg 1$ or when $P$ and $Q$ are far apart (i.e., when $\mu_q$ is large). The R{\'e}nyi-Donsker-Varadhan estimator only works well when $\mu_q$ and $\alpha$ are both not too large, but even in such cases the convex-conjugate R{\'e}nyi estimator generally performs better. We conjecture that this difference is due to the  presence of risk-sensitive terms in \req{eq:Renyi_var} which were eliminated in the new representation \req{eq:Renyi_LT_var_Mb_main}. We note that the NN for the convex-conjugate R{\'e}nyi estimator used the poly-softplus final layer, as we found the $-abs$ final layer to result in a significant percentage of failed runs (i.e., NaN outputs) but this issue did not arise when using poly-softplus. We do not show results for either DV-WCR or CC-WCR here as the exact divergence is infinite in this example.
\begin{figure}[ht]
\begin{minipage}[b]{0.43\linewidth}
  \centering
\includegraphics[scale=.50]{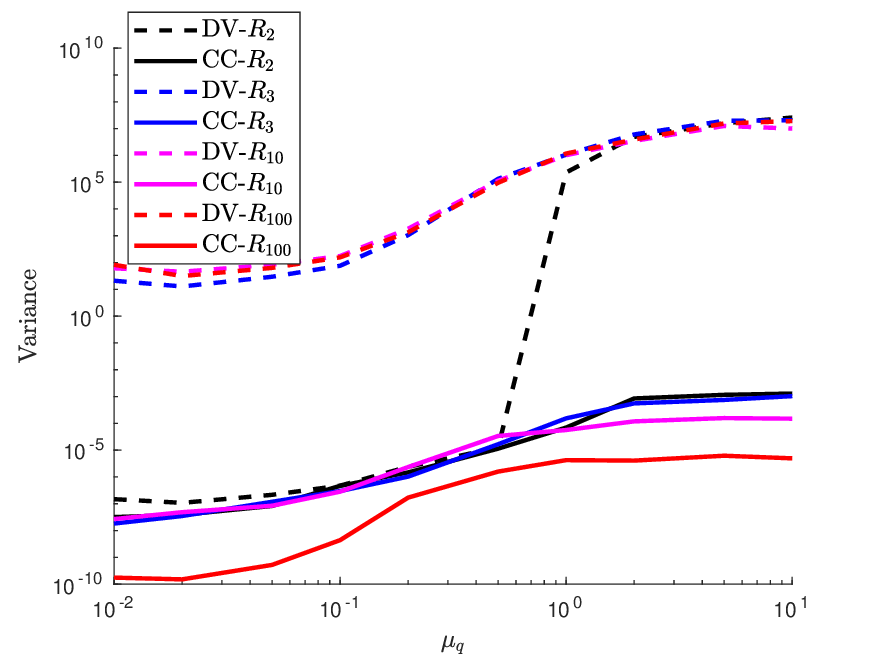} \end{minipage}
\hspace{0.5cm}
\begin{minipage}[b]{0.43\linewidth}
  \centering
\includegraphics[scale=.50]{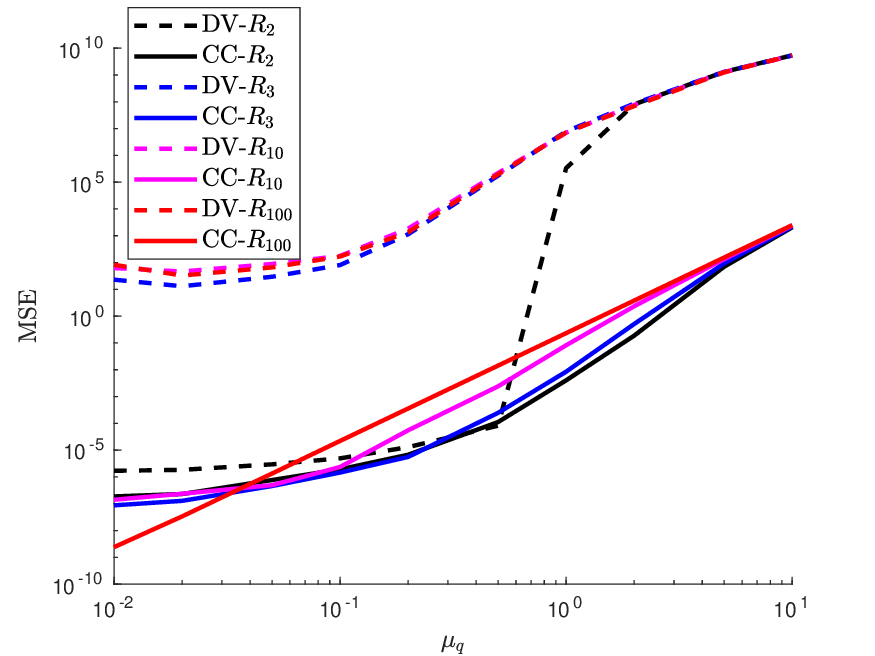} \end{minipage}
\caption{Variance and MSE of estimators of the classical R{\'e}nyi divergence between 1000-dimensional Gaussians. DV-$R_\alpha$ refers to R{\'e}nyi divergence estimators built using \req{eq:Renyi_var} while CC-$R_\alpha$ refers to estimators built using our new variational representation \req{eq:Renyi_LT_var_Mb_main}. We used a NN with one fully connected layer of 64 nodes, ReLU activations, and a poly-softplus final layer (for CC-R{\'e}nyi). We trained for 10000 epochs with a minibatch size of 500. The variance and MSE were computing using data from 50 independent runs. Note that the CC-R{\'e}nyi estimator has significantly  reduced variance and MSE compared to the DV-R{\'e}nyi estimator, even when $\alpha$ is large. Strikingly, the 1-D case exhibits the same behavior (see Figure \ref{fig:var_example_dim1} in Appendix \ref{app:var_example}), demonstrating that the DV-R{\'e}nyi estimator is unsuitable even in low dimensions. }\label{fig:var_example}
\end{figure}

\subsection{Detection of rare sub-populations in single-cell biological datasets}\label{sec:subpop}

A critical task in cancer assessment is the detection of rare sub-populations subsumed in the overall population of cells. The advent of affordable flow and mass cytometry technologies that perform single cell measurements opens a new direction for the analysis and comparison of high-dimensional cell distributions \cite{Shahi:sci:rep} via divergence estimation.
We consider single cell mass cytometry measurements on 16 bone marrow protein markers ($d=16$) coming from healthy and disease individuals with acute myeloid leukemia
\cite{Levine2015}. Following \cite{diffcyt}, we create two datasets: one with only healthy samples and another one with decreasing percentage of sick cells and compute several divergences. Considering the estimated divergence value as the score of a binary classifier, we can compute the ROC curve and the respective area under the ROC curve (AUC) for any pair of sample distributions. 
More specifically, true negatives correspond to the divergence values between two healthy datasets while true positives correspond to the divergence between a healthy and a diseased dataset. Thus, the AUC is 1.0 when the divergence estimates are completely separable while AUC is 0.5 when they completely overlap.
Table \ref{table:subpop:detection} reports the AUC values for the scaled IC-$\Gamma$-R\'enyi divergences \req{eq:D_alpha_Gamma_IC_def}, various levels of rarity and two sample sizes for the datasets. 
The best performance in the R\'enyi family is obtained for $\alpha=\infty$ using the IC-$\Gamma$-WCR variational formula \req{eq:D_inf_Gamma_var_formula}. IC-$\Gamma$-WCR also outperforms the  Wasserstein distance of first order in both   sample size regimes.  

 \begin{table}[h]
    \caption{AUC values (higher is better) for several divergences and various levels of rarity. The AUC values have been averaged from 50 independent runs. The neural discriminator has 2 hidden layers with 32 units each and ReLU activation. The $D_\alpha^{\Gamma,IC}$ divergences used the poly-softplus final layer.} \label{table:subpop:detection}
  \centering
  \resizebox{\textwidth}{!}{
  \begin{tabular}{cc|ccccccc|ccccc}
    \toprule
    \multicolumn{2}{c|}{Sample size} & \multicolumn{7}{c|}{100K} & \multicolumn{5}{c}{20K} \\
    \midrule
     \multicolumn{2}{c|}{Probability (\%)} & 0.1 & 0.2 & 0.3 & 0.4 & 0.5 & 1.0 &  10.0 & 0.1 & 0.3 & 0.5 & 1.0 & 10.0 \\
    \midrule
    \multirow{4}{*}{\rotatebox{90}{$D_\alpha^{\Gamma,IC}$}} & $\alpha=2$ & 0.51 & 0.55 & 0.62 & 0.64 & 0.70 & 0.92 & 1.00 & 0.48 & 0.58 & 0.58 & 0.60 & 1.00 \\
    & $\alpha=5$ & 0.66 & 0.70 & 0.71 & 0.72 & 0.74 & 0.80 & 1.00 & 0.32 & 0.37 & 0.43 & 0.38 & 0.91 \\
    & $\alpha=10$ & 0.57 & 0.50 & 0.62 & 0.64 & 0.49 & 0.59 & 1.00 & 0.48 & 0.48 & 0.43 & 0.47 & 0.74 \\
    & $\alpha=\infty$ & 0.64 & 0.89 & 0.96 & 0.99 & 1.00 & 1.00 & 1.00 & 0.58 & 0.71 & 0.79 & 0.91 & 1.00 \\
    \midrule
    \multicolumn{2}{c|}{Wasserstein} & 0.63 & 0.58 & 0.58 & 0.51 & 0.57 & 0.55 & 1.00 & 0.46 & 0.40 & 0.45 & 0.40 & 1.00 \\
\bottomrule
  \end{tabular}
  }
\end{table}

\subsection{IC-$\Gamma$-R{\'e}nyi GANs}
Finally, we study a pair of GAN examples. Here the goal is to learn a distribution $P$ using a family of generator distribution $Q_\psi\sim h_\psi(X)$ where $X$ is a noise source and $h_\psi$ is a family of neural networks parametrized by $\psi\in\Psi$, i.e., the goal is to solve 
\begin{align}
    \inf_{\psi\in\Psi}\widehat{D}_n(P\|Q_\psi)\,,
\end{align}
where $\widehat{D}_n$ is a divergence estimator of the form \req{eq:div_estimator}. In particular, we will study the GANs constructed from the newly introduced  IC-$\Gamma$-R{\'e}nyi and IC-$\Gamma$-WCR GANs and compare them with Wasserstein GAN \cite{wgan:gp,arjovsky2017wasserstein}.
\subsubsection{CIFAR-10}\label{ex:CIFAR10}
In Figure \ref{fig:C10_GAN} we demonstrate  improved performance of the IC-$\Gamma$-R{\'e}nyi and IC-$\Gamma$-WCR GANs, as compared to Wasserstein GAN with gradient penalty (WGAN-GP), on the CIFAR-10 dataset \cite{krizhevsky2009learning}.  The IC GANs also outperform R{\'e}nyi-DV GAN \req{eq:R_Gamma_DV}, as the latter is highly unstable when $\alpha>1$ and so the training generally encounters NaN after a small number of training epochs (hence we omit those results from the figure).   We use the same ResNet neural network architecture  as in \cite[Appendix F]{wgan:gp} and focus on  evaluating the effect of different divergences. Here we let $\Gamma$ be the set of 1-Lipschitz functions, implement via a gradient penalty. Note that $D_\infty^{\Gamma,IC}$ performs significantly better than $R_\alpha^{\Gamma,IC}$ with large $\alpha$, and the rescaled $D_\alpha^{\Gamma,IC}$-GAN performs better that $R_\alpha^{\Gamma,IC}$-GAN when $\alpha$ is large.

\begin{figure}[ht]

\begin{minipage}[b]{0.43\linewidth}
  \centering
\includegraphics[scale=.50]{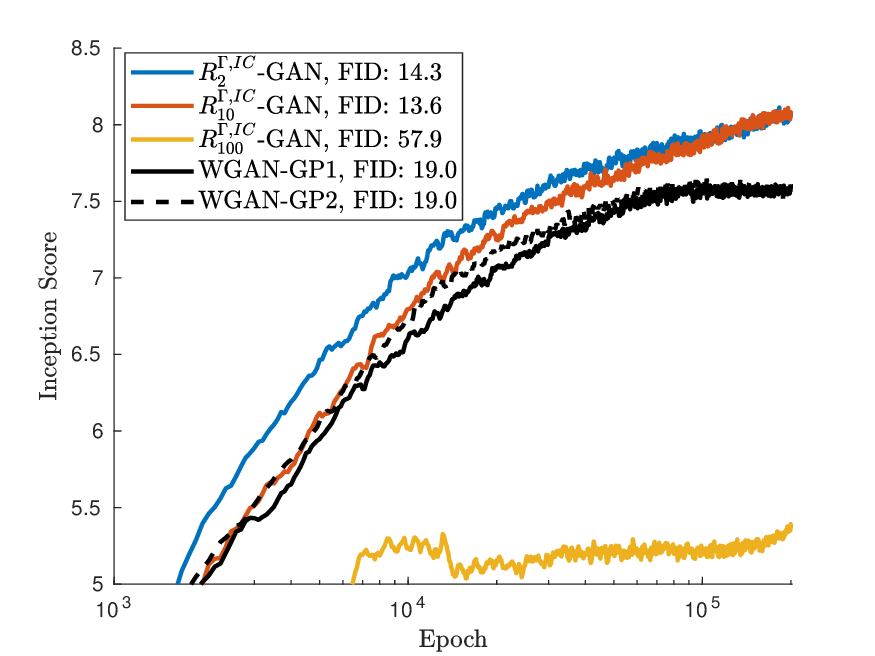} \end{minipage}
\hspace{0.5cm}
\begin{minipage}[b]{0.43\linewidth}
  \centering
\includegraphics[scale=.50]{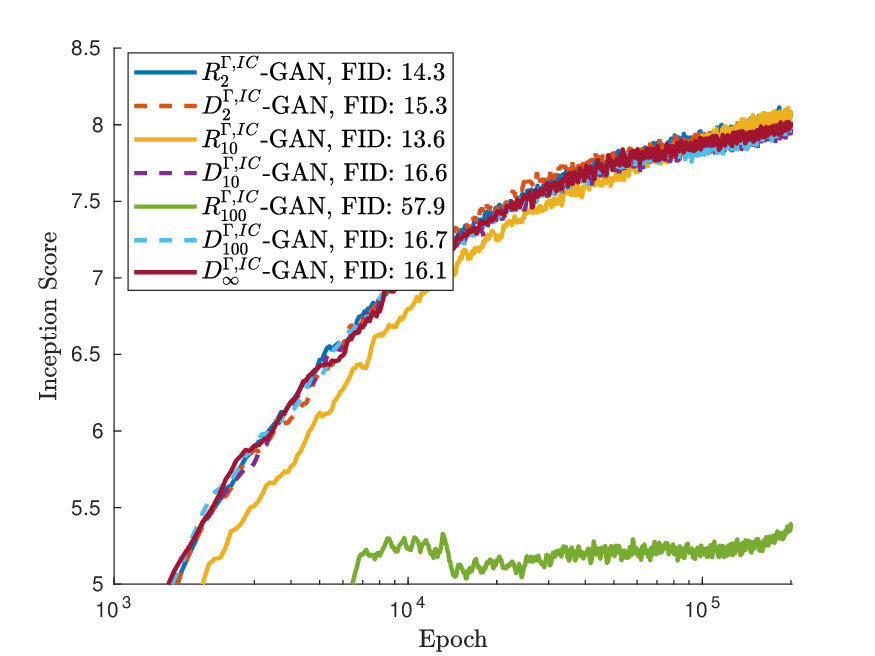} \end{minipage}
\caption{  Comparison between IC-$\Gamma$-R{\'e}nyi GAN, IC-$\Gamma$-WCR GAN, and WGAN-GP (both 1 and 2-sided) on the CIFAR-10 dataset.  Here we plot the inception score as a function of the number of training epochs (moving average over the last 5 data points, with results  averaged over 5 runs). We also show the averaged final FID score in the legend, computed using 50000 samples  from both $P$ and $Q$.  For the IC GANs we enforce negativity of the discriminator by using a final layer equal to $-\text{abs}$. The GANs were trained using the Adam optimizer with an initial learning rate of $0.0002$. The left pane shows that the IC-$\Gamma$-R{\'e}nyi GANs outperform WGAN while the right pane shows that GANs based on the rescaled $D_\alpha^{\Gamma,IC}$ divergences \req{eq:D_alpha_Gamma_IC_def} perform better when $\alpha$ is large, including in the $\alpha\to\infty$ limit, i.e., IC-$\Gamma$-WCR \req{eq:D_inf_Gamma_def}.  In both cases the IC GANs  outperform the  $\Gamma$-R{\'e}nyi-DV GANs with $\alpha>1$ \req{eq:R_Gamma_DV}; the latter fail to converge due to the presence of risk-sensitive terms. }\label{fig:C10_GAN}
\end{figure}

\subsubsection{RotMNIST}\label{sec:RotMNIST}
When learning a distribution $P$ that is invariant under a symmetry group (e.g., rotation invariance for images without preferred orientation) one can greatly increase performance by using a GAN that incorporates the symmetry information into the generator and the discriminator space $\Gamma$ \cite{EquivariantGAN}.  A theory of such symmetry-preserving GANs was developed in \cite{pmlr-v162-birrell22a} and the new divergences introduced in this paper satisfy the assumptions required to apply that theory. In Table \ref{tab:fid_rotmnist}  we  demonstrate this effectiveness on the RotMNIST dataset, obtained from randomly rotating the original  MNIST digit dataset \cite{lecun1998gradient}, resulting in an rotation-invariant distribution. Note that incorporating more symmetry information into the GAN (i.e., progressing down the rows of the table) results in greatly improved performance, especially in the low data regime.
  \begin{table}[h]
    \caption{The median of the FIDs (lower is better), calculated every 1,000 generator update for 20,000 iterations,  over three independent trials. The number of the training samples used for experiments varies from 1\% (600) to 10\% (6,000) of the RotMNIST training set. The NN structure and hyperparameters are the same as those used in Section 5.4 of \cite{pmlr-v162-birrell22a}. \texttt{Eqv G} (resp. \texttt{Inv D}) denotes that the symmetry information was incorporated into the generator (resp. discriminator) while \texttt{CNN} implies that a convolutional NN was used (without rotational symmetry).  $\Sigma$ denotes the rotation group used, where $C_n$ denotes rotations by being integer multiples of $2\pi/n$. } \label{tab:fid_rotmnist}  
  \centering
  \begin{tabular}{m{1em}cccccc}
    \toprule
     & Architecture &  1\%& 5\%& 10\% \\
    \midrule
    \rotatebox{90}{Reverse $R_2^{\Gamma,IC}$ }& \makecell{\texttt{CNN G\&D} \\
    \texttt{Eqv G} + \texttt{CNN D}, $\Sigma = C_4$\\    
    \texttt{CNN G} + \texttt{Inv D}, $\Sigma = C_4$\\
    \texttt{Eqv G} + \texttt{Inv D}, $\Sigma = C_4$ \\
    \texttt{Eqv G} + \texttt{Inv D}, $\Sigma = C_8$}&
\makecell{357\\464\\366\\151\\\textbf{114}}&\makecell{325\\271\\321\\105\\\textbf{71}}&\makecell{298\\263\\302\\89\\\textbf{62}}    \\
    \bottomrule
  \end{tabular}
\end{table}

\section{Proofs}\label{app:proofs}
In this section we provide a number of proofs that were omitted from the above text. Recall that $X$ denotes a compact metric space.
\begin{lemma}\label{lemma:R_IC_nondec_app}
Let $\Gamma\subset C(X)$  and  $P,Q\in\mathcal{P}(X)$.  Then $\alpha R_\alpha^{\Gamma/\alpha,IC}(P\|Q)$ is non-decreasing in $\alpha\in(0,1)\cup(1,\infty)$. If  $0\in\Gamma$ then $\alpha R_\alpha^{\Gamma,IC}(P\|Q)$ is also non-decreasing.
\end{lemma}
\begin{proof}
If $0\in\Gamma$ then $W^\Gamma\geq 0$, hence
\begin{align}
    \alpha R_\alpha^{\Gamma,IC}(P\|Q)= \inf_{\eta\in\mathcal{P}(X)}\{\alpha R_\alpha(P\|\eta)+\alpha W^\Gamma(Q,\eta)\}
\end{align}where both $\alpha\mapsto \alpha R_\alpha(P\|\eta)$ and $\alpha\mapsto \alpha W^\Gamma(Q,\eta)$ are non-decreasing.  Therefore the infimum is as well.  The proof for $\alpha\mapsto \alpha R_\alpha^{\Gamma/\alpha,IC}(P\|Q)$ is similar, though it doesn't require the assumption $0\in\Gamma$ due to the identity $\alpha W^{\Gamma/\alpha}=W^\Gamma$.
\end{proof}

Next we prove a key lemma that is used in our main result. First recall the definition
\begin{align}\label{eq:Lambda_def_app}
\Lambda_\alpha^P[g]\coloneqq\infty1_{g\not<0}-\left(\frac{1}{\alpha-1}\log\int |g|^{(\alpha-1)/\alpha} dP+\alpha^{-1}(\log\alpha +1)\right)1_{g<0}\,,\,\,\, g\in C(X).
\end{align}
\begin{lemma}\label{lemma:Lambda_cont_app}
$\Lambda_\alpha^P$ is convex and is continuous on $\{g\in C(X):g<0\}$, an open subset of $C(X)$.

\end{lemma}
\begin{proof}
First we prove convexity.  Let $g_0,g_1\in\{C(X):g<0\}$ and $\lambda\in(0,1)$. For $\alpha\in(0,1)$ we can use the inequality $\lambda a+(1-\lambda) b\geq a^\lambda b^{1-\lambda}$ for all $a,b>0$ to compute 
\begin{align}
&-\frac{1}{\alpha-1}\log\int |\lambda g_1+(1-\lambda)g_0|^{(\alpha-1)/\alpha} dP\leq -\frac{1}{\alpha-1}\log\int(|g_1|^\lambda|g_0|^{1-\lambda})^{(\alpha-1)/\alpha}dP\,.
\end{align}
Using H{\"o}lder's inequality with exponents $p=1/\lambda$, $q=1/(1-\lambda)$ we then obtain
\begin{align}
 &-\frac{1}{\alpha-1}\log\int(|g_1|^\lambda|g_0|^{1-\lambda})^{(\alpha-1)/\alpha}dP\\
\leq &-\frac{1}{\alpha-1}\log\left(\int |g_1|^{(\alpha-1)/\alpha}dP^{\lambda}\int |g_0|^{(\alpha-1)/\alpha}dP^{1-\lambda}\right)\notag\\
=&\lambda\left(-\frac{1}{\alpha-1}\log\int |g_1|^{(\alpha-1)/\alpha}dP\right)+(1-\lambda)\left(-\frac{1}{\alpha-1}\log\int |g_0|^{(\alpha-1)/\alpha}dP\right)\,.\notag
\end{align}
Therefore $g\mapsto -\frac{1}{\alpha-1}\log\int |g|^{(\alpha-1)/\alpha} dP$ is convex on $\{g<0\}$. This proves $\Lambda_\alpha^P$ is convex when $\alpha\in(0,1)$.  

Now suppose $\alpha>1$.  The map $t>0$, $t\mapsto t^{(\alpha-1)/\alpha}$ is concave and $-\log$ is decreasing and convex, hence
\begin{align}
&-\frac{1}{\alpha-1}\log\int|\lambda g_1+(1-\lambda)g_0|^{(\alpha-1)/\alpha}dP\\
\leq &-\frac{1}{\alpha-1}\log\left(\lambda \int |g_1|^{(\alpha-1)/\alpha}dP+(1-\lambda)\int |g_0|^{(\alpha-1)/\alpha}dP\right)\notag\\
\leq &\lambda\left(-\frac{1}{\alpha-1}\log\int|g_1|^{(\alpha-1)/\alpha}dP\right)+(1-\lambda)\left(-\frac{1}{\alpha-1}\log\int|g_0|^{(\alpha-1)/\alpha}dP\right)\,.\notag
\end{align}
This proves that $\Lambda_\alpha^P$ is also convex when $\alpha>1$.  Openness of $\{g<0\}$ follows from the assumption that $X$ is compact and so any strictly negative continuous function is strictly bounded away from zero.  Continuity on $\{g<0\}$ then follows from the dominated convergence theorem.
\end{proof}

Next we prove the limiting results from Theorem \ref{thm:IC_limits}.
\begin{theorem}\label{thm:IC_limits_app}
Let $\Gamma\subset C(X)$ be admissible, $P,Q\in\mathcal{P}(X)$, and $\alpha\in(0,1)\cup(1,\infty)$.  Then 
\begin{align}
     &\lim_{\delta\to0^+}\frac{1}{\delta}R_\alpha^{\delta\Gamma,IC}(P\|Q)=W^\Gamma(Q,P)\label{eq:scaling_limit_0}
\end{align}
and if $\Gamma$ is strictly admissible we have
    \begin{align}
    &\lim_{L\to\infty}R_\alpha^{L\Gamma,IC}(P\|Q)=R_\alpha(P\|Q)\,.\label{eq:scaling_limit_infinity}
\end{align}

\end{theorem}
\begin{proof}
It is straightforward to show that the scaled function spaces are admissible and $W^{c\Gamma}=cW^\Gamma$ for all $c>0$. First we prove \req{eq:scaling_limit_0}.  From the definition  \ref{eq:Gamma_renyi_def} we have
\begin{align}
    \delta^{-1}R_\alpha^{\delta\Gamma,IC}(P\|Q)=\inf_{\eta\in\mathcal{P}(X)}\{   \delta^{-1}R_\alpha(P\|\eta)+    W^\Gamma(Q,\eta)\}\leq W^\Gamma(Q,P)
\end{align}
and so $\delta^{-1}R_\alpha^{\delta\Gamma,IC}(P\|Q)$ is non-increasing in $\delta$. Therefore \begin{align}\label{eq:delta_limit_sup}
 \lim_{\delta\to 0^+}\delta^{-1}R_\alpha^{\delta\Gamma,IC}(P\|Q)=\sup_{\delta>0}   \delta^{-1}R_\alpha^{\delta\Gamma,IC}(P\|Q)
\end{align}
and
\begin{align}\label{eq:delta_limit_ub}
    \lim_{\delta\to 0^+}\delta^{-1}R_\alpha^{\delta\Gamma,IC}(P\|Q)\leq W^\Gamma(Q,P)\,.
\end{align}
We will assume this inequality is strict and derive a contradiction.  This assumption, together with \req{eq:delta_limit_sup}, implies $R_\alpha^{\delta\Gamma,IC}(P\|Q)<\infty$ for all $\delta>0$. Part (2) of Theorem \ref{thm:Renyi_inf_conv} then implies the existence of $\eta_{*,\delta}\in\mathcal{P}(X)$ such that
\begin{align}
\delta^{-1}R_\alpha^{\delta\Gamma,IC}(P\|Q)=\delta^{-1}R_\alpha(P\|\eta_{*,\delta})+ W^{\Gamma}(Q,\eta_{*,\delta})\geq W^\Gamma(Q,\eta_{*,\delta})\,.
\end{align}
Take a sequence $\delta_n\to 0^+$. We have assumed $X$ is compact, hence $\mathcal{P}(X)$ is also compact and so there exists a weakly convergent subsequence $\eta_{*,\delta_{n_j}}\to \eta_*$. From the variational formulas \req{eq:Renyi_var}  and \req{def:IPM}  we see that $R_\alpha(P\|\cdot)$ and $W^\Gamma(Q,\cdot)$ are lower semicontinuous, hence  $\liminf_j W^\Gamma(Q,\eta_{*,\delta_{n_j}})\geq W^\Gamma(Q,\eta_{*})$ and
\begin{align}
   R_\alpha(P\|\eta_{*})\leq& \liminf_j R_\alpha(P\|\eta_{*,\delta_{n_j}})\leq \liminf_j \delta_{n_j}(\delta_{n_j}^{-1}R_\alpha(P\|\eta_{*,\delta_{n_j}})+W^\Gamma(Q,\eta_{*,\delta_{n_j}}))\\
   =&\liminf_j\delta_{n_j}(\delta_{n_j}^{-1} R_\alpha^{\delta_{n_j}\Gamma,IC}(P\|Q))=0\,,
\end{align}
where the last equality follows from the assumed strictness of the inequality \req{eq:delta_limit_sup}. Therefore the divergence property for the classical R{\'e}nyi divergences implies $R_\alpha(P\|\eta_*)=0$ and $P=\eta_*$. Combining the above results we obtain
\begin{align}
  \lim_{\delta\to 0^+}\delta^{-1}R_\alpha^{\delta\Gamma,IC}(P\|Q)=&\lim_{j\to\infty}  \delta_{n_j}^{-1}R_\alpha^{\delta_{n_j}\Gamma,IC}(P\|Q)\geq\liminf_j W^\Gamma(Q,\eta_{*,\delta_{n_j}})\\
  \geq& W^\Gamma(Q,\eta_*)=W^\Gamma(Q,P)\,.\notag
\end{align}
This contradicts \req{eq:delta_limit_ub} and therefore we have proven the equality \req{eq:scaling_limit_0}.

Now we assume $\Gamma$ is strictly admissible and will prove \req{eq:scaling_limit_infinity} via similar reasoning. From the definition  \ref{eq:Gamma_renyi_def} we see that
\begin{align}
    R_\alpha^{L\Gamma,IC}(P\|Q)=\inf_{\eta\in\mathcal{P}(X)}\{R_\alpha(P\|\eta)+LW^{\Gamma}(Q,\eta)\}\leq R_\alpha(P\|Q)
\end{align}
and $R_\alpha^{L\Gamma,IC}$
is non-decreasing in $L$. Hence $\lim_{L\to\infty}R_\alpha^{L\Gamma,IC}(P\|Q)=\sup_{L>0}R_\alpha^{L\Gamma,IC}(P\|Q)$ and
\begin{align}\label{eq:L_limit_ub}
    \lim_{L\to\infty}R_\alpha^{L\Gamma,IC}(P\|Q)\leq R_\alpha(P\|Q)\,.
\end{align}
Suppose this inequality is strict. Then $R_\alpha^{L\Gamma,IC}(P\|Q)<\infty$ for all $L$ and we can use part (2) of Theorem \ref{thm:Renyi_inf_conv} to conclude there exists $\eta_{*,L}\in\mathcal{P}(X)$ such that
\begin{align}
  R_\alpha^{L\Gamma,IC}(P\|Q)=R_\alpha(P\|\eta_{*,L})+LW^{\Gamma}(Q,\eta_{*,L})\,.  
\end{align}
Take $L_n\to\infty$. Compactness of $\mathcal{P}(X)$ implies the existence of a weakly convergent subsequence $\eta_{*,j}\coloneqq\eta_{*,L_{n_j}}\to \eta_*\in\mathcal{P}(X)$. Lower semicontinuity of $R_\alpha(P\|\cdot)$ and $W^\Gamma(Q,\cdot)$ imply $\liminf_j R_\alpha(P\|\eta_{*,j})\geq R_\alpha(P\|\eta_*)$  and
\begin{align}
 W^\Gamma(Q,\eta_*)\leq &    \liminf_jW^\Gamma(Q,\eta_{*,j})=\liminf_j L_{n_j}^{-1}W^{L_{n_j}\Gamma}(Q,\eta_{*,j})
 \\
 \leq& \liminf_j L_{n_j}^{-1} R_\alpha^{L_{n_j}\Gamma,IC}(P\|Q)=0\,,\notag
    \end{align}
where the last equality follows from the assumed strictness of the inequality  \req{eq:L_limit_ub}. Therefore $W^\Gamma(Q,\eta_*)=0$. $\Gamma$ is strictly admissible, hence $Q=\eta_*$ (see the proof of part (7) of Theorem \ref{thm:Renyi_inf_conv}). Combining these results we see that
\begin{align}
    \lim_{L\to\infty}R_\alpha^{L\Gamma,IC}(P\|Q)=&\lim_jR_\alpha^{L_{n_j}\Gamma,IC}(P\|Q)=\lim_j(R_\alpha(P\|\eta_{*,L_{n_j}})+L_{n_j}W^{\Gamma}(Q,\eta_{*,L_{n_j}}))\\
    \geq& \liminf_j R_\alpha(P\|\eta_{*,j})\geq R_\alpha(P\|\eta_*)=R_\alpha(P\|Q)\,.\notag
\end{align}
This contradicts the assumed strictness of the inequality \req{eq:L_limit_ub} and hence \req{eq:L_limit_ub} is an equality. This completes the proof.

\end{proof}

Next we prove Theorem \ref{thm:alpha_1_limit}, regarding the $\alpha\to 1$ limit of the IC-$\Gamma$-R{\'e}nyi divergences.
\begin{theorem}\label{thm:alpha_1_limit_app}
Let $\Gamma\subset C(X)$ be admissible and $P,Q\in\mathcal{P}(X)$.  Then
\begin{align}
\lim_{\alpha\to1^+}R_\alpha^{\Gamma,IC}(P\|Q)=&\inf_{\substack{\eta\in\mathcal{P}(X):\\\exists\beta>1,R_\beta(P\|\eta)<\infty}}\{R(P\|\eta)+W^\Gamma(Q,\eta)\}\,,\label{eq:IC_rev_KL_plus_app}\\
    \lim_{\alpha\to 1^-}R_\alpha^{\Gamma,IC}(P\|Q)=&\inf_{\eta\in\mathcal{P}(X)}\{R(P\|\eta)+W^\Gamma(Q,\eta)\}\label{eq:IC_rev_KL_app}\\
    =&\sup_{g\in \Gamma:g<0}\{\int gdQ+\int \log|g|dP\}+1\,.\label{eq:rev_KL_var_formula_minus_app}
\end{align}
\end{theorem}
\begin{proof}
Lemma \ref{lemma:R_IC_nondec_app} implies $\alpha\mapsto\alpha R_\alpha^{\Gamma,IC}(P\|Q)$ is non-decreasing on $(1,\infty)$, therefore
\begin{align}
\lim_{\alpha\to 1^+}\alpha R_\alpha^{\Gamma,IC}(P\|Q)=&\inf_{\alpha>1}\alpha R_\alpha^{\Gamma,IC}(P\|Q)\\
=&\inf_{\alpha>1}\inf_{\eta\in\mathcal{P}(X)}\{\alpha R_\alpha(P\|\eta)+\alpha W^\Gamma(Q,\eta)\}\notag\\
=&\inf_{\eta\in\mathcal{P}(X)}\inf_{\alpha>1}\{\alpha R_\alpha(P\|\eta)+\alpha W^\Gamma(Q,\eta)\}\notag\\
=&\inf_{\eta\in\mathcal{P}(X)}\{\lim_{\alpha\to 1^+}\alpha R_\alpha(P\|\eta)+ W^\Gamma(Q,\eta)\}\notag\,.
\end{align}
From \cite{van2014renyi} we have
\begin{align}
\lim_{\alpha\to 1^+} R_\alpha(P\|\eta)=\begin{cases}
R(P\|\eta) & \text{ if } \exists \beta>1, R_\beta(P\|\eta)<\infty\\
\infty &\text{ otherwise}
\end{cases}
\end{align}
and so we can conclude
\begin{align}
\lim_{\alpha\to 1^+} R_\alpha^{\Gamma,IC}(P\|Q)=\lim_{\alpha\to 1^+}\alpha R_\alpha^{\Gamma,IC}(P\|Q)=\inf_{\substack{\eta\in P(X):\\\exists \beta>1,R_\beta(P\|\eta)<\infty}}\{R(P\|\eta)+W^\Gamma(Q,\eta)\}\,.
\end{align}
This proves \req{eq:IC_rev_KL_plus_app}.

Now we compute the limit as $\alpha\to 1^-$. Note that the limit exists due to the fact that $\alpha\mapsto\alpha R_\alpha^{\Gamma,IC}(P\|Q)$ is non-decreasing. From the definition \req{eq:Gamma_renyi_def}, for all $\eta\in\mathcal{P}(X)$ we have
\begin{align}
\lim_{\alpha\to 1^-}R_\alpha^{\Gamma,IC}(P\|Q)\leq \lim_{\alpha\to 1^-}R_\alpha(P\|\eta)+W^\Gamma(Q,\eta)=R(P\|\eta)+W^\Gamma(Q,\eta)\,.
\end{align}
Here we used the fact that $\lim_{\alpha\to 1^-}R_\alpha(P\|\eta)=R(P\|\eta)$ (see \cite{van2014renyi}). Maximizing over $\eta$ then gives
\begin{align}
\lim_{\alpha\to 1^-}R_\alpha^{\Gamma,IC}(P\|Q)\leq \inf_{\eta\in\mathcal{P}(X)}\{R(P\|\eta)+W^\Gamma(Q,\eta)\}\,.
\end{align}
To prove the reverse inequality, use part 1 of Theorem \ref{thm:Renyi_inf_conv} to compute
\begin{align}
\lim_{\alpha\to 1^-}R_\alpha^{\Gamma,IC}(P\|Q)=&\lim_{\alpha\to 1^-}\alpha R_\alpha^{\Gamma,IC}(P\|Q)\\
=&\lim_{\alpha\to 1^-}\sup_{g\in \Gamma:g<0}\left\{\alpha\int gdQ +\frac{\alpha}{\alpha-1}\log\int |g|^{(\alpha-1)/\alpha} dP+\log\alpha +1\right\}\notag\\
\geq &\int gdQ +\lim_{\alpha\to 1^-}\frac{\alpha}{\alpha-1}\log\int |g|^{(\alpha-1)/\alpha} dP +1\notag\\
=&\int gdQ +\frac{d}{dy}|_{y=0}\log\int e^{y\log|g|} dP +1\notag\\
=&\int gdQ +\int \log|g|dP+1\notag
\end{align}
for all $g\in \Gamma$, $g<0$. Therefore, maximizing over $g$ gives
\begin{align}\label{eq:limit_1_minus_lb}
\lim_{\alpha\to 1^-}R_\alpha^{\Gamma,IC}(P\|Q)\geq&\sup_{g\in\Gamma:g<0}\left\{\int gdQ +\int \log|g|dP\right\}+1\,.
\end{align}
We now use Fenchel-Rockafellar duality (Theorem 4.4.3 in \cite{borwein2006techniques}) to compute the dual variational representation of the right hand side of \req{eq:limit_1_minus_lb}. Define $F,G:C(X)\to (-\infty,\infty]$ by $F[g]=\infty1_{g\not<0}-\int \log|g|dP$ and $G[g]=\infty1_{g\not\in\Gamma}-E_Q[g]$.  It is stratightforward to show that $F$ and $G$ are convex, $F[-1]<\infty$, $G[-1]<\infty$, and $F$ is continuous at $-1$. Therefore
\begin{align}
\inf_{g\in C(X)}\{F[g]+G[g]\}=\sup_{\eta\in E^*}\{-F^*(-\eta)-G^*(\eta)\}\,,
\end{align}
i.e.
\begin{align}
\sup_{g\in \Gamma:g<0}\{E_Q[g]+\int\log|g|dP\}=\inf_{\eta\in M(X)}\{F^*(\eta)+W^\Gamma(Q,\eta)\}\,,
\end{align}
where $F^*(\eta)=\sup_{g\in C(X):g<0}\{\int gd\eta+\int \log|g|dP\}$.  Now we show the infimum can be restricted to $\eta\in\mathcal{P}(X)$:  If $\eta(X)\neq 1$ then by taking $g=\pm n$ we find
\begin{align}
W^\Gamma(Q,\eta)\geq  n |Q(X)-\eta(X)|\to \infty
\end{align}
as $n\to\infty$.  Therefore $W^\Gamma(Q,\eta)=\infty$ if $\eta(X)\neq 1$. 

Now suppose $\eta\in M(X)$ is not positive.  Take a measurable set $A$ with $\eta(A)<0$. By Lusin's theorem, for all $\epsilon>0$ there exists a closed set $E_\epsilon\subset X$ and a continuous function $g_\epsilon\in C(X)$ such that $|\eta|(E_\epsilon^c)<\epsilon$, $0\leq g_\epsilon\leq 1$, and $g_\epsilon|_{E_\epsilon}=1_A$.  Define $g_{n,\epsilon}=-ng_\epsilon-1$, $n\in\mathbb{Z}^+$.  Then $g_{n,\epsilon}\in\{g\in C(X):g<0\}$, hence
\begin{align}
F^*(\eta)\geq& \int g_{n,\epsilon}d\eta+\int \log|g_{n,\epsilon}|dP\\
=& \int -ng_\epsilon-1d\eta+\int \log(ng_\epsilon+1)dP\notag\\
\geq & n(|\eta(A)|-\int (g_\epsilon-1_A)1_{E_\epsilon^c}d\eta)-\eta(X)\notag\\
\geq & n(|\eta(A)|-\epsilon)-\eta(X)\,.\notag
\end{align}
Letting $\epsilon=|\eta(A)|/2$ and taking $n\to\infty$ gives  $F^*(\eta)=\infty$.  Therefore we  conclude 
\begin{align}
\inf_{\eta\in M(X)}\{F^*(\eta)+W^\Gamma(Q,\eta)\}=\inf_{\eta\in \mathcal{P}(X)}\{F^*(\eta)+W^\Gamma(Q,\eta)\}\,.
\end{align}
To evaluate $F^*(\eta)$ for $\eta\in\mathcal{P}(X)$ we make a change of variables $g=-\exp(h-1)$, $h\in C(X)$ to obtain
\begin{align}
F^*(\eta)=&\sup_{h\in C(X)}\{\int hdP-\int e^{h-1}d\eta\}-1=R(P\|\eta)-1\,.
\end{align}
Here we used the Legendre-transform variational representation of the KL divergence; see equation (1) in \cite{9737725} with $f(x)=x\log(x)$.  Combining these results we obtain
\begin{align}
\inf_{\eta\in\mathcal{P}(X)}\{R(P\|\eta)+W^\Gamma(Q,\eta)\}\geq& \lim_{\alpha\to 1^-}R_\alpha^{\Gamma,IC}(P\|Q)\\
\geq&\sup_{g\in\Gamma:g<0}\left\{\int gdQ +\int \log|g|dP\right\}+1\notag\\
=&\inf_{\eta\in M(X)}\{F^*(\eta)+W^\Gamma(Q,\eta)\}+1\notag\\
=&\inf_{\eta\in \mathcal{P}(X)}\{R(P\|\eta)+W^\Gamma(Q,\eta)\}\notag\,.
\end{align}
This completes the proof.
\end{proof}

Now we prove Theorem \ref{thm:alpha_infinity_limit}, regarding the $\alpha\to\infty$ limit of the IC-$\Gamma$-R{\'e}nyi divergences.
\begin{theorem}\label{thm:alpha_infinity_limit_app} Let $\Gamma\subset C(X)$ be admissible and $P,Q\in\mathcal{P}(X)$. Then
\begin{align}
    \lim_{\alpha\to\infty}\alpha R_\alpha^{\Gamma/\alpha,IC}(P\|Q)=&\inf_{\eta\in P(X)}\{D_\infty(P\|\eta)+W^\Gamma(Q,\eta)\}\\
    =&\sup_{g\in \Gamma:g<0}\left\{\int gdQ+\log\int|g|dP\right\}+1\,.
\end{align}
\end{theorem}
\begin{proof}
First note that
\begin{align}
\alpha R_\alpha^{\Gamma/\alpha,IC}(P\|Q)=\inf_{\eta\in\mathcal{P}(X)}\{\alpha R_\alpha(P\|\eta)+W^\Gamma(Q,\eta)\}
\end{align}
is nondecreasing in $\alpha$, therefore for  $\eta\in\mathcal{P}(X)$ we have
\begin{align}
\lim_{\alpha\to\infty}\alpha R_\alpha^{\Gamma/\alpha,IC}(P\|Q)=&\sup_{\alpha>1}\alpha R_\alpha^{\Gamma/\alpha,IC}(P\|Q)\\
\leq&\sup_{\alpha>1}\{\alpha R_\alpha(P\|\eta)+W^\Gamma(Q,\eta)\}\notag\\
=&D_\infty(P\|\eta)+W^\Gamma(Q,\eta)\,.\notag
\end{align}
 Maximizing over $\eta$ gives the upper bound
\begin{align}
\lim_{\alpha\to\infty}\alpha R_\alpha^{\Gamma/\alpha,IC}(P\|Q)\leq  \inf_{\eta\in\mathcal{P}(X)}\{D_\infty(P\|\eta)+W^\Gamma(Q,\eta)\}\,.
\end{align}

To prove the reverse inequality, use the variational formula \req{eq:Renyi_inf_conv} to write
\begin{align}
\alpha R_\alpha^{\Gamma/\alpha,IC}(P\|Q)=&\alpha\sup_{g\in \Gamma:g<0}\left\{\int g/\alpha dQ +\frac{1}{\alpha-1}\log\int |g/\alpha|^{(\alpha-1)/\alpha} dP\right\}+\log\alpha +1\\
=&\sup_{g\in \Gamma:g<0}\left\{\int g dQ +\frac{\alpha}{\alpha-1}\log \int |g|^{(\alpha-1)/\alpha} dP\right\}+1\,.\notag
\end{align}
Therefore, for all $g\in\Gamma$, $g<0$ we can use the dominated convergence theorem to compute
\begin{align}
\lim_{\alpha\to\infty}\alpha R_\alpha^{\Gamma/\alpha,IC}(P\|Q)\geq&\int g dQ +\lim_{\alpha\to\infty}\frac{\alpha}{\alpha-1}\log \int |g|^{(\alpha-1)/\alpha} dP+1\\
=&\int g dQ +\log\int|g|dP+1\,.\notag
\end{align}
Maximizing over $g$ then gives
\begin{align}\label{eq:D_infty_Gamma_lb}
\lim_{\alpha\to\infty}\alpha R_\alpha^{\Gamma/\alpha,IC}(P\|Q)\geq&\sup_{g\in\Gamma:g<0}\left\{\int g dQ +\log\int|g|dP\right\}+1\,.
\end{align}
Next we use the Fenchel-Rockafellar duality to derive a dual formulation of the right hand side of \req{eq:D_infty_Gamma_lb}.  Define $G,F:C(X)\to(-\infty,\infty]$, $G[g]=\infty 1_{g\not\in \Gamma}-E_Q[g]$, $F[g]=\infty1_{g\not< 0}-\log\int|g|dP$.  It is straightforward to prove that $G,F$ are convex and $G[-1]<\infty, F[-1]<\infty$ and $F$ is continuous at $-1$.  Therefore Fenchel-Rockafellar duality implies
\begin{align}
\inf_{g\in C(X)}\{F[g]+G[g]\}=\sup_{\eta\in C(X)^*}\{-F^*[-\eta]-G^*[\eta]\}\,,
\end{align}
i.e.
\begin{align}
\sup_{g\in \Gamma:g<0}\left\{E_Q[g]+\log\int|g|dP\right\}=\inf_{\eta\in M(X)}\{F^*[\eta]+W^\Gamma(Q,\eta)\}\,,
\end{align}
where $F^*[\eta]=\sup_{g\in C(X):g<0}\{\int gd\eta+\log\int|g|dP\}$. We now prove that the infimum over $M(X)$ can be restricted to $\mathcal{P}(X)$. First suppose $\eta(X)\neq 1$. Then, because $\Gamma$ contains the constant functions,  we have
\begin{align}
W^\Gamma(Q,\eta)\geq\pm n(1-\eta(X))\to \infty
\end{align}
as $n\to\infty$ for appropriate choice of sign.  Therefore $W^\Gamma(Q,\eta)=\infty$ when $\eta(X)\neq 1$. 

Now suppose $\eta\in M(X)$ is not positive.  Take a measurable set $A$ with $\eta(A)<0$. By Lusin's theorem, for all $\epsilon>0$ there exists a closed set $E_\epsilon\subset X$ and a continuous function $g_\epsilon\in C(X)$ such that $|\eta|(E_\epsilon^c)<\epsilon$, $0\leq g_\epsilon\leq 1$, and $g_\epsilon|_{E_\epsilon}=1_A$.  Define $g_{n,\epsilon}=-ng_\epsilon-1$, $n\in\mathbb{Z}^+$.  Then $g_{n,\epsilon}\in\{g\in C(X):g<0\}$, hence
\begin{align}
F^*[\eta]\geq& \int g_{n,\epsilon}d\eta+\log\int |g_{n,\epsilon}|dP\\
=&n(|\eta(A)| -\int(g_\epsilon-1_A) 1_{E_\epsilon^c}d\eta) -\eta(X)+\log(n\int g_\epsilon dP+1)\notag\\
\geq&n(|\eta(A)| -\epsilon) -\eta(X)\,.\notag
\end{align}
Letting $\epsilon=|\eta(A)|/2$ and then taking $n\to\infty$ we see that $F^*[\eta]=\infty$ when $\eta$ is not positive. Together  these results imply
\begin{align}
\inf_{\eta\in M(X)}\{F^*[\eta]+W^\Gamma(Q,\eta)\}=\inf_{\eta\in \mathcal{P}(X)}\{F^*[\eta]+W^\Gamma(Q,\eta)\}\,.
\end{align}
Finally, using Theorem \ref{thm:worst_case_regret_var} we see that
\begin{align}
F^*[\eta]+1=\sup_{g\in C(X):g<0}\{\int gd\eta+\log\int|g|dP\}+1=D_\infty(P\|\eta)
\end{align}
for all $\eta\in\mathcal{P}(X)$. Combining these results gives
\begin{align}
\lim_{\alpha\to\infty}\alpha R_\alpha^{\Gamma/\alpha,IC}(P\|Q)\geq&\sup_{g\in\Gamma:g<0}\left\{\int g dQ +\log\int|g|dP\right\}+1\\
=&\inf_{\eta\in M(X)}\{F^*[\eta]+W^\Gamma(Q,\eta)\}+1\notag\\
=&\inf_{\eta\in \mathcal{P}(X)}\{D_\infty(P\|\eta)+W^\Gamma(Q,\eta)\}\geq\lim_{\alpha\to\infty}\alpha R_\alpha^{\Gamma/\alpha,IC}(P\|Q)\,.\notag
\end{align}
This completes the proof.
\end{proof}

Finally, we prove Theorem \ref{thm:IC_data_processing}, the data-processing inequality for the IC-$\Gamma$-R{\'e}nyi divergences. See the paragraph above Theorem \ref{thm:IC_data_processing} for definitions of the notation.
\begin{theorem}[Data Processing Inequality]\label{thm:IC_data_processing_app}
Let $\alpha\in(0,1)\cup(1,\infty)$, $Q,P\in\mathcal{P}(X)$, and $K$ be a probability kernel from $X$ to $Y$ such that $K[g]\in C(X)$ for all $g\in C(X,Y)$.  
\begin{enumerate}
    \item  If  $\Gamma\subset C(Y)$ is admissible then
    \begin{align}\label{eq:data_proc1_app}
        R_\alpha^{\Gamma,IC}\left(K[P]\|K[Q]\right)\leq R_\alpha^{K[\Gamma],IC}(P\|Q)\,.
    \end{align}
    \item  If  $\Gamma\subset C(X\times Y)$ is admissible then
    \begin{align}\label{eq:data_proc2_app}
        R_\alpha^{\Gamma,IC}\left(P\otimes K\|Q\otimes K\right)\leq R_\alpha^{K[\Gamma],IC}(P\|Q)\,.
        \end{align}
\end{enumerate}
\end{theorem}
\begin{proof}
It is straightforward to show that admissiblility of $\Gamma$ implies admissibility of $K[\Gamma]$.  Hence we can write
\begin{align}\label{eq:data_proc_tmp}
R_\alpha^{K[\Gamma],IC}(P\|Q)=&\sup_{\tilde g\in K[\Gamma]:\tilde g<0}\left\{\int \tilde g dQ +\frac{1}{\alpha-1}\log\int |\tilde g|^{(\alpha-1)/\alpha} dP\right\}+\alpha^{-1}(\log\alpha +1)\\
\geq &\sup_{ g\in \Gamma: g<0}\left\{\int  K[g] dQ +\frac{1}{\alpha-1}\log\int |K[ g]|^{(\alpha-1)/\alpha} dP\right\}+\alpha^{-1}(\log\alpha +1)\,.\notag
\end{align}
Using Jensen's inequality we can derive
\begin{align}
    &|\int g(y)K_x(dy)|^{(\alpha-1)/\alpha}\leq \int |g(y)|^{(\alpha-1)/\alpha}K_x(dy)\,,\,\,\,\alpha\in(0,1)\,,\label{eq:data_proc_tmp2}\\
       &|\int g(y)K_x(dy)|^{(\alpha-1)/\alpha}\geq \int |g(y)|^{(\alpha-1)/\alpha}K_x(dy)\,,\,\,\,\alpha>1\,.\label{eq:data_proc_tmp3}
       \end{align}
       Combining \req{eq:data_proc_tmp} - \req{eq:data_proc_tmp3} with the monotonicity of $y\mapsto\frac{1}{\alpha-1}\log(y)$ we arrive at \req{eq:data_proc1_app}. The proof of \req{eq:data_proc2_app} is similar.
\end{proof}

\subsection{Further remarks regarding the proof of Theorem \ref{thm:Renyi_inf_conv}}\label{sec:dual_var_proof}
 The Fenchel-Rockafellar Theorem applies under two different sets of assumptions:  the first assumes both mappings are lower semicontinuous (LSC)  while the second applies when one mapping is continuous at a point where both are finite. The mapping $\Lambda_\alpha^P$, as defined by \req{eq:Lambda_def_app} and appearing in  \req{eq:Renyi_LT}, is {\bf not} LSC but it is continuous on its domain, hence we used the second version of Fenchel-Rockafellar in our proof of Theorem \ref{thm:Renyi_inf_conv}. For $\alpha>1$ one could  alternatively redefine $\Lambda_\alpha^P$ along the boundary of $\{g<0\}$ to make it LSC while still maintaining the relation \req{eq:Renyi_LT} and thereby utilize the first version of Fenchel-Rockafellar.  This alternative approach is also amenable to extending the theorem to non-compact spaces, using the methods from \cite{Dupuis:Mao,JMLR:v23:21-0100}. However, these methods do not apply to  $\alpha\in(0,1)$. With this in mind, in order to provide a simple unified treatment of all $\alpha\in(0,1)\cup(1,\infty)$ we structured our proof around the second version of the Fenchel-Rockafellar Theorem.
 
 Despite the fact that $\Lambda_\alpha^P$  is not LSC, the Fenchel-Rockafellar Theorem does imply that convex duality holds at all points of continuity in the domain, i.e., one has
\begin{align}
\Lambda_\alpha^P[g]=\sup_{\eta\in M(X)}\{\int gd\eta-R_\alpha(P\|\eta)\}\,\, \text{ for all $g<0$}\,,
\end{align}
but this duality formula doesn't necessarily hold if $g\not<0$. Here,  $R_\alpha(P\|\eta)$ for general $\eta\in M(X)$  is defined via the variational formula
\begin{align}
    R_\alpha(P\|\eta)\coloneqq (\Lambda_\alpha^P)^*[\eta]= &\sup_{g\in C(X)}\{\int gd\eta-\Lambda_\alpha^P[g]\}
\end{align}
    and one can rewrite this in terms of the classical R{\'e}nyi divergence as follows
\begin{align}
      R_\alpha(P\|\eta)=&\begin{cases}
    \infty \text{ if } \eta\not\geq 0 \text{ or }\eta=0\,,\\
    R_\alpha(P\|\frac{\eta}{\|\eta\|})-\frac{1}{\alpha}\log\|\eta\| \text{ if $\eta$ is a nonzero positive measure.}    \end{cases}
\end{align}

\section*{Acknowledgments}
The research of J.B., M.K. and L.R.-B. was partially supported by the Air Force Office of Scientific Research (AFOSR) under the grant FA9550-21-1-0354. The research of M. K. and L.R.-B. was partially supported by the National Science Foundation (NSF) under the grants DMS-2008970 and TRIPODS CISE-1934846. The research of P.D. was partially supported by the 
NSF under the grant DMS-1904992 and by the AFOSR under the grant FA-9550-21-1-0354. The work of Y.P. was partially supported by the Hellenic Foundation for Research and Innovation (HFRI) through the “Second Call for HFRI Research Projects to support Faculty Members and Researchers” under Project 4753. This work was performed in part using high performance computing equipment obtained under a grant from the Collaborative R\&D Fund managed by the Massachusetts Technology Collaborative.

\appendix
\section{Variance of R{\'e}nyi estimators}\label{app:var_example}
Here we compare the DV-R{\'e}nyi and CC-R{\'e}nyi estimators on the   Gaussian test problem from Section \ref{sec:var_example}, except in lower dimensions (1-D and 100-D).  Qualitatively, the behavior is similar.  In particular, it is striking that the DV-R{\'e}nyi estimator performs extremely poorly even in the 1-D case (see Figure \ref{fig:var_example_dim1}) while the CC-R{\'e}nyi estimator has much lower variance and MSE when the separation between the distributions becomes larger (i.e., as $\mu_q$ increases).

\begin{figure}[ht]
\begin{minipage}[b]{0.43\linewidth}
  \centering
\includegraphics[scale=.50]{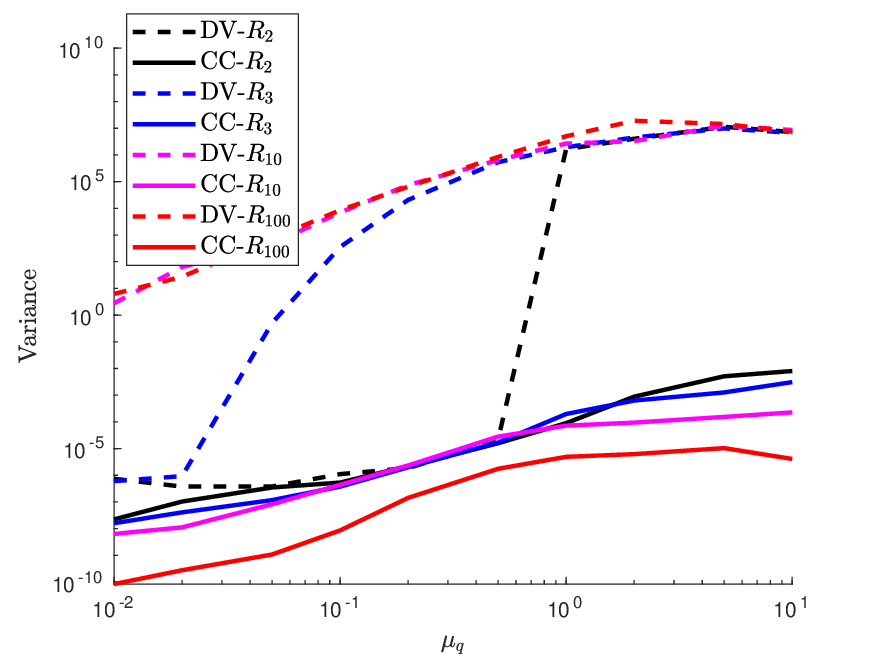} \end{minipage}
\hspace{0.5cm}
\begin{minipage}[b]{0.43\linewidth}
  \centering
\includegraphics[scale=.50]{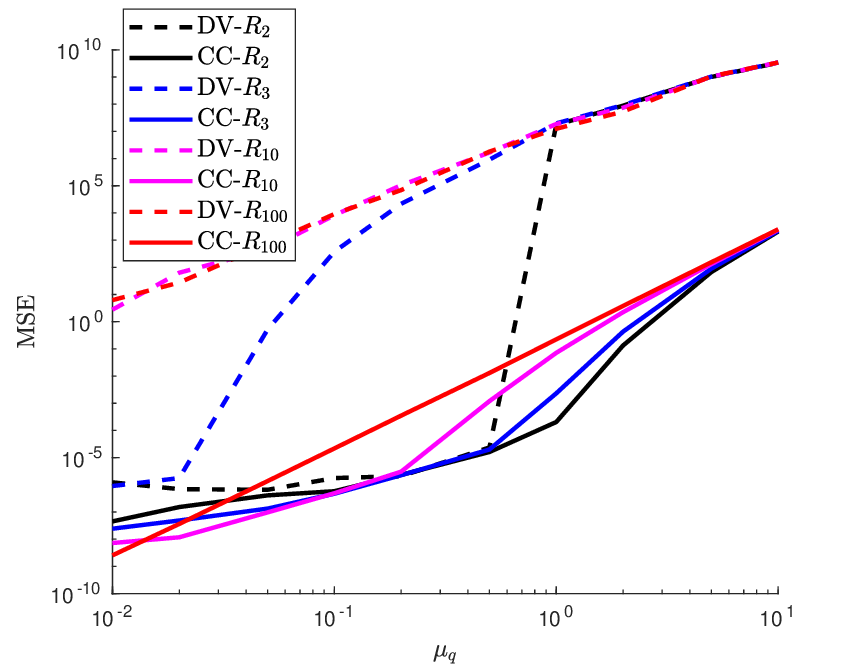} \end{minipage}
\caption{Variance and MSE of estimators of the classical R{\'e}nyi divergence between 1-dimensional Gaussians. DV-$R_\alpha$ refers to R{\'e}nyi divergence estimators built using \req{eq:Renyi_var} while CC-$R_\alpha$ refers to estimators built using our new variational representation \req{eq:Renyi_LT_var_Mb_main}. We used a NN with one fully connected layer of 64 nodes, ReLU activations, and a poly-softplus final layer (for CC-R{\'e}nyi). We trained for 10000 epochs with a minibatch size of 500. The variance and MSE were computing using data from 50 independent runs. Note that the CC-R{\'e}nyi estimator has significantly  reduced variance and MSE compared to the DV-R{\'e}nyi estimator, even when $\alpha$ is large.}\label{fig:var_example_dim1}
\end{figure}

\begin{figure}[ht]
\begin{minipage}[b]{0.43\linewidth}
  \centering
\includegraphics[scale=.50]{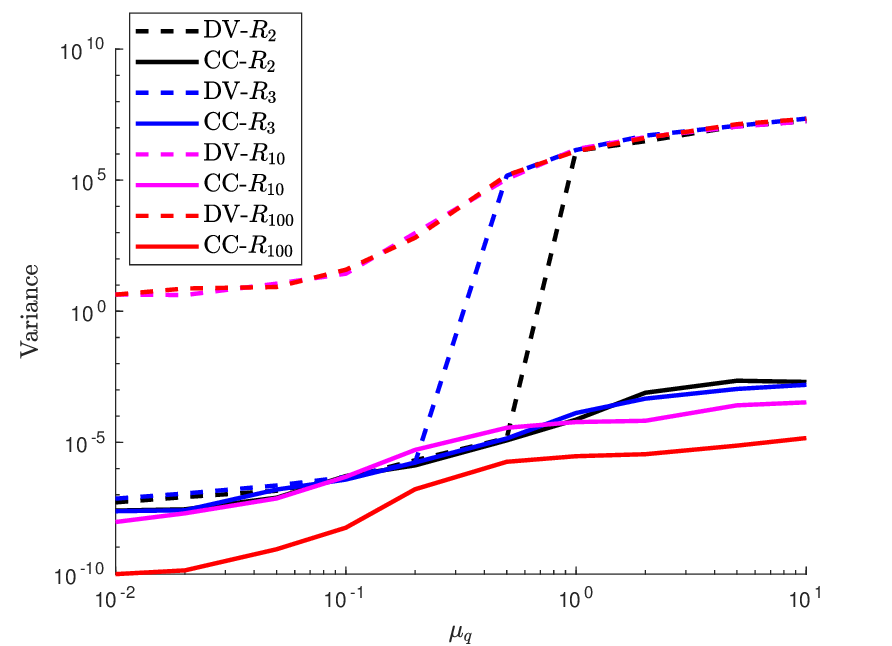} \end{minipage}
\hspace{0.5cm}
\begin{minipage}[b]{0.43\linewidth}
  \centering
\includegraphics[scale=.50]{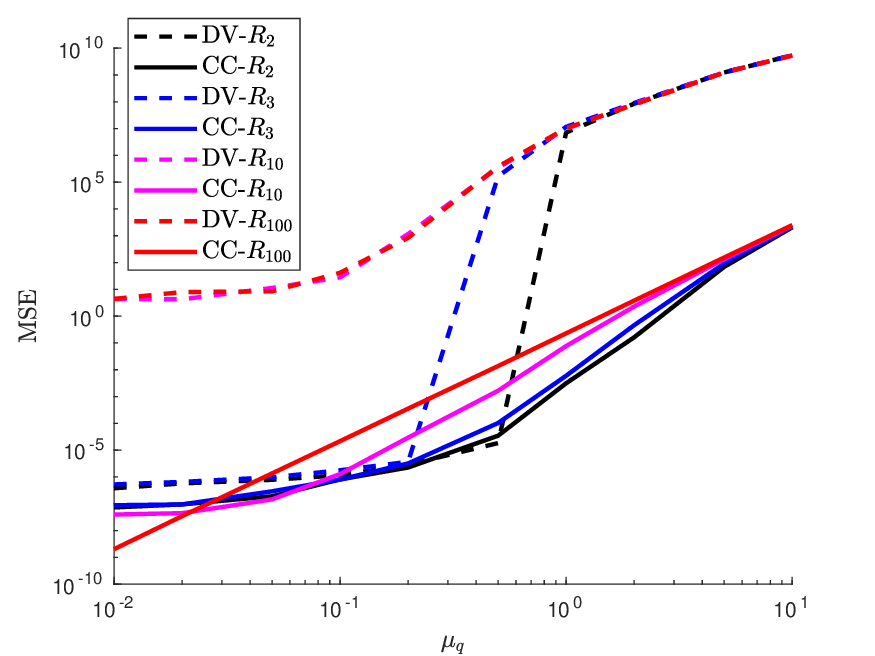} \end{minipage}
\caption{Variance and MSE of estimators of the classical R{\'e}nyi divergence between 100-dimensional Gaussians. DV-$R_\alpha$ refers to R{\'e}nyi divergence estimators built using \req{eq:Renyi_var} while CC-$R_\alpha$ refers to estimators built using our new variational representation \req{eq:Renyi_LT_var_Mb_main}. We used a NN with one fully connected layer of 64 nodes, ReLU activations, and a poly-softplus final layer (for CC-R{\'e}nyi). We trained for 10000 epochs with a minibatch size of 500. The variance and MSE were computing using data from 50 independent runs. Again, the CC-R{\'e}nyi estimator has significantly  reduced variance and MSE compared to the DV-R{\'e}nyi estimator, even when $\alpha$ is large.}\label{fig:var_example_dim100}
\end{figure}

\end{document}